\documentclass[final,12pt]{clear2023} 

\usepackage{booktabs} 
\usepackage{multirow}
\usepackage{wrapfig}

\title[An Algorithm and Complexity Results for Causal Unit Selection]{An Algorithm and Complexity Results for Causal Unit Selection}

\usepackage{times}

\clearauthor{%
 \Name{Haiying Huang} \Email{hhaiying@cs.ucla.edu}\\
 \addr UCLA Computer Science Department
 \AND
 \Name{Adnan Darwiche} \Email{darwiche@cs.ucla.edu}\\
 \addr UCLA Computer Science Department%
}

\usepackage{algorithm}
\usepackage{algpseudocode}

\newcommand\shrink[1]{}

\def\n(#1){\bar{#1}}

\def\pr{{\it Pr}}

\def\C{{\bf C}}

\def\E{{\bf E}}
\def\e{{\bf e}}
\def\Pa{{\bf P}}
\def\pa{{\bf p}}
\def\P{{\bf P}}
\def\p{{\bf p}}

\def\U{{\bf U}}
\def\u{{\bf u}}
\def\V{{\bf V}}
\def\v{{\bf v}}

\def\W{{\bf W}}
\def\w{{\bf w}}
\def\X{{\bf X}}
\def\x{{\bf x}}
\def\Y{{\bf Y}}
\def\y{{\bf y}}
\def\Z{{\bf Z}}
\def\z{{\bf z}}
\def\C{{\bf C}}

\def\S{{\bf S}}

\def\F{\mathcal{F}}

\def\eql(#1,#2){{#1\!\!=\!#2}}
\def\NP{\text{NP}}
\def\PP{\text{PP}}
\def\SP{\#\text{P}}

\DeclareMathOperator*{\dsep}{dsep}
\DeclareMathOperator*{\union}{\bigcup}

\DeclareMathOperator{\map}{MAP}
\DeclareMathOperator{\rmap}{RMAP}
\DeclareMathOperator{\nei}{neigh}

\newcommand\dup[2]{{[#1]}^{#2}}

\def\eql(#1,#2){{#1\!=\!#2}}
\newcommand\name[1]{\ensuremath{\mathsf{#1}}}

\newcommand{\argmax}{\operatornamewithlimits{arg max}}

\newcommand\bracket[1]{\left[{#1}\right]}

\def\clap#1{\hbox to 0pt{\hss#1\hss}}

\newcommand{\Ainput}[1]{\hspace*{\algorithmicindent} \textbf{Input: }#1 }
\newcommand{\Aoutput}[1]{\hspace*{\algorithmicindent} \textbf{Output: }#1 }

\begin{document}

\maketitle

\begin{abstract}%
  The unit selection problem aims to identify
objects, called units, that are most likely to exhibit 
a desired mode of behavior when subjected to stimuli 
(e.g., customers who are about to churn but would 
change their mind if encouraged). Unit selection with
counterfactual objective functions was introduced
relatively recently with existing work focusing
on bounding a specific class of objective functions, called the benefit functions, based on observational and interventional data---assuming
a fully specified model is not available to 
evaluate these functions. We complement this line of work
by proposing the first exact algorithm for finding 
optimal units given a broad class of causal
objective functions and a fully specified  structural causal model (SCM). We show that unit
selection under this class of objective functions is $\NP^\PP$-complete but is $\NP$-complete when
unit variables correspond to all exogenous variables
in the SCM. We also provide treewidth-based 
complexity bounds on our proposed algorithm while 
relating it to a well-known algorithm 
for Maximum a Posteriori (MAP) inference.

\end{abstract}

\begin{keywords}%
  unit selection, structural causal models, counterfactual reasoning%
\end{keywords}

\section{Introduction}
A theory of causality has emerged over the last few decades based on two parallel hierarchies, 
an {\em information hierarchy} and a {\em reasoning hierarchy,} often called the {\em causal hierarchy}~\citep{pearl18,Bareinboim20211OP}. 
On the reasoning side, this theory has crystalized three levels of reasoning with increased sophistication and proximity 
to human reasoning: associational, interventional and counterfactual, which are exemplified by the following canonical probabilities.
{\em Associational} \(\pr(y | x)\): probability of \(y\) given that \(x\) was observed (e.g., probability that a patient has a flu given they have a fever).
{\em Interventional} \(\pr(y_x)\): probability of \(y\) given 
that \(x\) was established by an intervention,
which is different from \(\pr(y | x)\) (e.g.,
seeing the barometer 
fall tells us about the weather but moving the barometer needle won't bring rain).
{\em Counterfactual} \(\pr(y_x | y', x')\): probability of \(y\) if we were to establish
\(x\) given that neither \(x\) nor \(y\) are true (e.g., probability that a patient who
did not take a vaccine and died would have recovered had they been vaccinated).
On the information side, these forms of reasoning 
require different
levels of knowledge, encoded as associational, causal and functional (mechanistic) models, 
with each class of models containing more information than the preceding one. In the framework of probabilistic graphical models
\citep{PGMbook}, such knowledge is encoded 
by Bayesian networks~\citep{Pearl88b,DarwicheBook09}, 
causal Bayesian networks~\citep{pearl00b,PetersBook,SpirtesBook} and functional Bayesian networks~\citep{uai/BalkeP95}
also known as {\em structural causal models} (SCMs).

One utility of this theory has been recently 
crystallized through the
{\em unit selection problem} introduced by \cite{ijcai/LiP19} who motivated it using 
the problem of selecting customers to target by an encouragement offer for renewing a subscription.
Let \(c\) denote the characteristics of a customer, \(x\) denote encouragement and \(y\) denote renewal. One
can use counterfactuals to describe the different types
of customers. 
A responder (\(y_x, y_{x'}'\)) would renew a subscription if encouraged but would not renew otherwise.
An always-taker (\(y_x,y_{x'}\)) would always renew regardless of encouragement.
An always-denier (\(y_x',y_{x'}'\)) would always not renew regardless of encouragement.
A contrarian (\(y_{x'},y_x'\)) would not renew if encouraged but would renew otherwise. 
One can then identify
customers of interest by optimizing an expression, called
a {\em benefit function} in \citep{ijcai/LiP19}, that includes counterfactual 
probabilities. In this example, the benefit function
has the form
\(\beta \pr(\mbox{responder} | c) + \gamma \pr(\mbox{always-taker} | c) + 
\theta \pr(\mbox{always-denier} | c)  + \delta \pr(\mbox{contrarian} | c)\)
where \(\beta, \gamma, \theta, \delta\) are corresponding benefits. In other words, 
one can use this expression to score customers with characteristics \(c\) so
the most promising ones can be selected for an encouragement offer.
When the above benefit function is contrasted with 
classical loss functions (for example, ones used to train neural networks), one sees a fundamental role for 
counterfactual reasoning as it gives us an ability to
distinguish between objects (e.g., people, situations) depending on how they respond to a stimulus.
This distinction sets apart counterfactual reasoning 
(third level of the causal hierarchy)
from the more common, but less refined, associational reasoning (first level). 
It also sets it apart from interventional reasoning 
(second level) which is also not sufficient to make such distinctions.

Existing work on unit selection has focused on a very
practical setting in which only the structure of an SCM
is available together with some observational and 
experimental 
data~\citep{ijcai/LiP19,aaai/LiP22,corr/LiP22a,corr/LiP22b,corr/LiJSP22}. Such data is usually not sufficient to obtain
a fully specified SCM so one cannot obtain point values of the
benefit function. Recent work has therefore focused on
bounding probabilities of causation while tightening these bounds
as much as possible \citep{dawid2017,pearl:etal21-r505},
but with less attention dedicated to optimizing benefits 
based on these bounds; see \citep{corr/LiJSP22,ang:etal22-r519}
for a notable exception.
In this paper, 
we complement this line of work by studying the unit selection
problem from a different and computational direction. 
We are particularly
interested in applying unit selection to structured units 
(e.g., decisions, policies, people, situations, regions, activities) that correspond to instantiations of multiple variables (called unit variables). 
We assume a fully specified SCM so we can obtain 
point values for any {\em causal objective function} 
as discussed in \sectionref{sec:background}. By a causal
objective function we mean any expression involving quantities from any level of the causal hierarchy (observational, interventional and counterfactual). This allows us to 
seek units that satisfy a broad class of conditions.
Examples include: Which combination of activities are most effective to address a particular humanitarian need (human suffering, disease, hunger, privation)? Which regions should be focused on to reduce population movements among refugees? What incentive policy would keep customers engaged for the longest time? 
We then consider a particular but broad class of causal
objective functions in \sectionref{sec:objective-function}
and formally define the computational problem of finding units that
optimize these functions. We dedicate 
\sectionref{sec:complexity of US} to studying the complexity of unit selection in this setting where we show it has
the same complexity as the classical 
{\em Maximum a Posteriori (MAP)} problem.
We then provide an exact algorithm for solving the unit
optimization problem in \sectionref{sec:VE-RMAP} by 
reducing it to a
new problem that we call {\em Reverse-MAP.} We further
characterize the complexity of our proposed algorithm
using the notion of treewidth and provide some analysis
on how its complexity can change depending on the specific
objective function we use. We finally close with
some concluding remarks in \sectionref{sec:conclusion}.
Some proofs of our results are included in the main paper,
the remaining ones can be found in the appendix.

\section{Counterfactual Queries on Structural Causal Models}
\label{sec:background}
We review {\em structural causal models} (SCMs) 
in this section since the unit selection problem is defined on 
these models; see \citep{galles1998axiomatic,halpern2000axiomatizing}
for a comprehensive exposition. 
We use uppercase letters (e.g., $X$) to denote
variables and lowercase letters (e.g., $x$) to denote their states.
We use bold uppercase letters (e.g., $\X$) to denote sets of
variables and bold lowercase letters (e.g., $\x$) to denote their instantiations. 
The states of a binary variable \(X\) are denoted \(x\) and \(x'\).
We also write \(x \in \x\) to mean that variable \(X\)
has state \(x\) in instantiation \(\x\) of variables \(\X\).

An SCM has three components. First, a directed acyclic graph
with its nodes representing variables. Root nodes are called
{\em exogenous} and internal nodes are called {\em endogenous.}
Second, a probability distribution \(\theta(U)\) for each exogenous variable \(U\) in the model. Third, for each endogenous variable
\(V\) with parents \(\P\), the SCM has an equation,
called a {\em structural equation}, which specifies 
a state for \(V\) for each instantiation \(\p\) of its parents \(\P\). 
Let \(\U/\V\) be the exogenous/endogenous variables in an SCM. The distribution \(\pr(\U,\V)\) specified
by the SCM is as follows: \(\pr(\u,\v)=\prod_{u\in\u}\theta(u)\) 
if \(\V=\v\) is implied by \(\U=\u\) and the structural equations;
otherwise, \(\pr(\u,\v) = 0\).

SCMs are a special type of Bayesian networks~\citep{pearlBNbook,DarwicheBook09}
which require a conditional probability table (CPT) for each
node in the network. In particular, for node \(V\) with parents
\(\P\), the CPT specifies the conditional distributions \(\pr(V|\P)\). 
A structural equation can be encoded as a 
CPT which satisfies \(\pr(v|\p) \in \{0,1\}\)
for all \(v\) and \(\p\). Such a CPT is said to be {\em functional} and this is why SCMs are sometimes called {\em functional
Bayesian networks.}

A Bayesian network can only be used to compute {\em observational} probabilities such as \(\pr(y|x)\) which is 
the probability of \(Y=y\) given that we {\em observed} \(X=x\).
An SCM can also be used to compute {\em interventional probabilities}
such as \(\pr(y_x)\) which is the probability of \(Y=y\) after 
{\em setting} \(X=x\). An SCM can further be used to compute 
{\em counterfactual probabilities} such as 
\(\pr(y_x, y'_{x'}| e)\) which is the probability of 
(\(Y=y\) after setting \(X=x\)
and \(Y=y'\) after setting \(X = x'\)) in a situation where 
we observe \(E=e\).\footnote{The class of {\em causal Bayesian
networks} sits between Bayesian networks and functional
Bayesian networks as it can be used to compute observational
and interventional probabilities but not counterfactual 
ones~\citep{pearl2000models}.}
We are particularly interested in this form of counterfactual probabilities as they will be used as
ingredients in our objective functions. We next 
show how to compute such a counterfactual probability on an SCM
by computing an observational probability on an auxiliary 
model. This will be essential for the constructions used later 
in the paper.

Consider the counterfactual probability \(\pr(y_x, y_{x'}' | x,y)\) on the SCM in \figureref{fig:base-triplet}. This
query has three conflicting components: \(y_x\), \(y_{x'}'\) and \((x,y)\).
The first two involve conflicting actions (\(x\) and \(x'\)).
Moreover, the actions and outcomes in the first two components
conflict with the observation in the third component (\(x,y\)). This is why
computing counterfactual probabilities usually requires
an auxiliary model that incorporates multiple worlds (real
and imaginary) that all share the same causal mechanisms
(exogenous variables). For the counterfactual queries we
are interested in, an auxiliary model with three worlds
will suffice as we discuss next.

Given an SCM \(G\), its {\em triplet model} is another SCM
constructed by having three copies \(G^1\), \(G^2\) and \(G^3\)
of \(G\) and then joining them so they share their exogenous variables; see Figure~\ref{fig:plain-triplet}. If \(X\) is a variable in \(G^1\),
we will use \([X]\) to denote its copy in \(G^2\) and
\([[X]]\) to denote its copy in \(G^3\). 
A triplet model is a special case of {\em parallel worlds
models} \citep{ijcai/AvinSP05} which also include {\em twin 
models} \citep{aaai/BalkeP94}.\footnote{Twin models are
sufficient to evaluate counterfactual probabilities like
\(\pr(y_{x'}' | x,y)\) and \(\pr(y_x', y_{x'})\)
but not ones like \(\pr(y_x, y'_{x'} | e)\) which we are
interested in; see also \citep{amai/TianP00,pearl2000models}.}
We can now compute the counterfactual probability 
\(\pr(y_x, y_{x'}' | x,y)\) on SCM \(G\) by operating on
the triplet model as follows. 
First, we mutilate copies \(G^2\) and \(G^3\) in the triplet
model by removing the edges pointing into variables 
$[X]$ and $[[X]]$ and setting $[X] = x$ and $[[X]] = x'$
(since we are intervening on these variables).
The result is a {\em mutilated triplet model} shown in  \figureref{fig:mut-triplet}. We can then evaluate \(\pr(y_x, y_{x'}' | x,y)\) on the SCM \(G\) by computing the
observational probability 
$\pr([y], [[y']] \mid [x], [[x']], x, y)$ on the 
mutilated triplet model. Intuitively, the triplet model
can be viewed as capturing three worlds \(G^1\), \(G^2\)
and \(G^3\). World \(G^1\) captures the observation 
\(x,y\); world \(G^2\) captures the intervention \(X=x\),
and world \(G^3\) captures the intervention \(X=x'\).
This above treatment can be directly generalized to 
counterfactual queries
of the form \(\pr(\y_\x,\w_\v|\e)\) where \(\E,\X,\Y,\V,\W\)
are sets of variables. It is precisely this class of
counterfactual queries that we shall use in the rest of 
the paper, starting with the next section.

\begin{figure}[t]
\floatconts
{fig:triplet}
{\caption{Reducing the counterfactual probability 
\(\pr(y_x, y_{x'}' | x,y)\)
on the model in (a) to an observational probability
$\pr([y], [[y']] \mid [x], [[x']], x, y)$ 
on the model in (c).}}
{
\subfigure[SCM]{\label{fig:base-triplet}
    \includegraphics[width=0.18\textwidth]{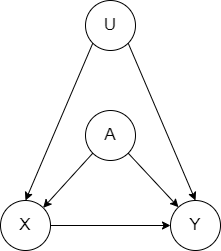}
} \quad
\subfigure[triplet model]{\label{fig:plain-triplet}
    \includegraphics[width=0.35\textwidth]{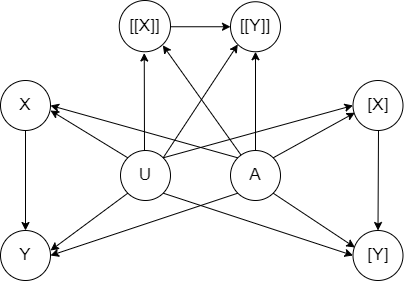}
}\quad
\subfigure[mutilated triplet model]{\label{fig:mut-triplet}
    \includegraphics[width=0.35\textwidth]{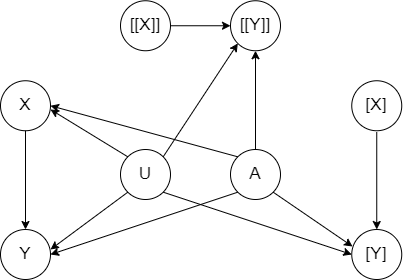}
}
}
\end{figure}

\section{Causal Objective Functions and Unit Selection}
\label{sec:unit-selection}
\label{sec:objective-function}
A causal objective function can be any
expression that involves observational, interventional
or counterfactual probabilities where the goal of unit
selection is to find objects (units)
that optimize this function. However, inspired by~\citep{ijcai/LiP19},
our treatment will be based on a specific class of causal objective functions which is a linear combination of counterfactual probabilities
of the form \(\pr({\y^i}_{\x^i},{\w^i}_{\v^i}|\e^i, \u)\) where \(i=1,\ldots,n\). 
We call \(\U\) the {\em unit variables} since our goal 
is to find instantiations \(\u\) of these 
variables (i.e., units) that optimize the objective function.\footnote{An anonymous reviewer pointed out that the term ``unit'' is often used to designate the unit of analysis; that is, the entity that is characterized by random variables. For example, the unit in many medical studies is the the patient, the unit in many management studies is the company, and the unit in many studies of crime rates is the city or municipality. In this context, ``unit selection'' could be assumed to involve selecting the unit of analysis
which is different from our use of the term in this paper.}
Variables \(\X^i\V^i\) represent treatments, variables
\(\Y^i\W^i\) represent outcomes, 
\((\X^i\cup\V^i) \cap (\Y^i\cup\W^i) = \emptyset\),
and variables \(\E^i\) represent evidence.
Unit variables are shared by all components of the objective
function but each component can have its own treatment,
outcome and evidence variables.

We will further assume that unit variables \(\U\) are exogenous in the SCM (i.e., root variables)
while treatment, outcome and evidence variables
are endogenous. However, not all exogenous variables need to be unit variables. This is consistent with the assumption in
\citep{aaai/LiP22} that unit variables (also called
{\em characteristics}) cannot be 
descendants of treatment or outcome variables.
This leads us to objective functions of the 
following form:\footnote{The conditions we place on
weights \(w_i\) are meant for convenience and they
are not restrictive.}
\begin{equation}
    L(\u) = \sum_{i=1}^n w_i\cdot\pr(\y^i_{\x^i},\w^i_{\v^i}|\e^i, \u) \quad \text{where \ } w_i \geq 0 
    \text{ \ and \ }\sum_{i=1}^n w_i = 1
    \label{eqn:objective-function}
\end{equation}

We can now formally define the unit selection inference problem
on structural causal models.

\begin{definition}[Unit Selection]
\label{def:unit selection}
Given an SCM \(G\), a subset \(\U\) of its variables,
and an objective function \(L(\u)\) such as
\equationref{eqn:objective-function}, the
unit selection inference problem is to compute \(\argmax_\u L(\u)\).
\end{definition}

The benefit function discussed in \citep{ijcai/LiP19} has
the following form:
\begin{equation}
     L(u) = 
     \beta \pr(y_x,  y_{x'}'|u) +  
     \gamma\pr(y_x,  y_{x'}|u) + 
     \theta\pr(y_x', y_{x'}'|u) + 
     \delta\pr(y_x', y_{x'}|u)
     \label{eqn:benefit-ang}
\end{equation}
This class of objective functions falls as a special case of \equationref{eqn:objective-function}
by setting \(n=4\), \(\E^i = \emptyset\),
\(\X^i = \V^i = \{X\}\) and \(\Y^i = \W^i = \{Y\}\) 
for \(i=1,\ldots,4\),
where \(X,Y\) are binary variables. That is, each component \(i\)
of the objective function uses the same single, treatment 
variable \(X\) and the same single, outcome variable \(Y\).
A more general form was 
proposed in~\citep{corr/LiP22a} in which treatment 
\(X\) has values \(x_1,\ldots,x_m\) and
outcome \(Y\) has values \(y_1, \ldots, y_k\) so the
objective function can have up to \(k^m\) components, 
each corresponding to a distinct response
type such as
\(\pr({y_2}_{x_1}, {y_1}_{x_2},{y_1}_{x_3},{y_3}_{x_4}, {y_2}_{x_5}|u)\) when \(k=3\) and \(m=5\).
This class of objective functions is more general
than \equationref{eqn:objective-function} in that
it allows one to express more response types but it assumes one
treatment variable and one outcome variable.
The class of objective functions we consider 
in \equationref{eqn:objective-function} allows
compound treatments and outcomes. It also allows
us to seek units from a particular group.
For example, if \(A\) and \(B\) are two medications 
(binary treatments)
and \(T\) and \(P\) refer to high temperature and high
blood pressure (binary outcomes), 
and \(E\) is the age group with values 
\(e_1, \ldots, e_4\), then the objective function
can include terms such as
\(\pr({(t,p')}_{a,b}\ ,\ {(t',p')}_{a',b}\mid e_3,u)\),
which is the probability that a member of the third
age group would have high temperature and
normal blood pressure if administered both medications and
would have normal temperature and blood pressure if 
administered only the second medication. Moreover, 
since the objective function components can have different treatment and outcome variables, one can select units based 
on their responses to distinct stimuli (e.g., effect of
one type of encouragement on membership renewal 
and the simultaneous effect of another type of
encouragement on increased purchases).\footnote{Going beyond
the form in \equationref{eqn:objective-function}, one can
use causal objective functions with more general ingredients, such as: the
probability of a patient being a responder given they 
are not a contrarian, \(\pr(y_x, {y}_{x'}' | \neg ({y}_x', y_{x'}))\); or the probability that a patient would not have had a stroke if they were on a diet $(y'_d)$ or had exercised $(y_e')$ given that they did neither $(d', e')$, i.e., \(\pr({y}_d' \vee {y}_e' | y, {d}', {e}')\).
Such general quantities have not been treated in the 
literature but some discussions 
have argued for their significance and treated some special cases; e.g.,~\citep{DisjunctiveActions}. 
}

\section{The Complexity of Unit Selection}
\label{sec:complexity of US}

We show next that unit selection is $\NP^\PP$-complete 
for the class of causal objective
functions given in \equationref{eqn:objective-function}.
We also show that this problem is $\NP$-complete when
unit variables correspond to all exogenous variables
in the SCM.\footnote{For a discussion of complexity classes that are relevant to
Bayesian network inference, 
see~\citep{shimony1994MPE} on the MPE decision problem being $\NP$-complete,
and~\citep{park2002map,ParkD04} on the MAP decision problem being $\NP^{\PP}$-complete.
\cite{ai/Roth96} shows that computing node marginals in a Bayesian network 
is $\SP$-complete.
For a textbook discussion of these complexity results, see~\citep[Ch.~11]{DarwicheBook09}.
}
We start 
by providing an efficient reduction from unit selection into a variant of the well-known MAP inference problem, which we call Reverse-MAP. 
We then follow 
by studying the complexity of Reverse-MAP and unit
selection. 

\label{sec:unit-selection-reduction}
Recall that our goal is to find units \(\u\) that maximize
the value \(L(\u)\) of the objective function. The first
step in solving this optimization problem is to be able
to evaluate the objective \(L(\u)\).
We next show a construction which allows us to evaluate
\(L(\u)\) by evaluating a single observational probability involving unit variables $\U$ but on an extended and mutilated model. This construction will serve two purposes. First,
it will permit us to characterize the complexity of unit selection
when using objective functions in the form of \equationref{eqn:objective-function}. 
Second, we will later use the construction to
develop a specific algorithm for solving the unit 
selection problem using these objective functions.

Consider each term $\pr(\y^i_{\x^i}, \w^i_{\v^i}|\e^i, \u)$ in \equationref{eqn:objective-function}. We reviewed in \sectionref{sec:background} how this quantity can be 
reduced to a classical conditional probability 
on a triplet model $G^i$. 
The next step is to encode a linear combination of these conditional probabilities as a conditional probability
on some model \(G'\). This is done using the following construction.

\begin{definition}[Objective Model]
\label{def:objective model}
\label{def:n-network}
Consider an SCM $G$ with parameters \(\theta\) and
the objective function $L$ in \equationref{eqn:objective-function}.
The objective model $G’$ for 
$\langle G,L\rangle$ has parameters \(\theta'\) 
and constructed as follows:
\begin{enumerate}
\item Construct a triplet model $G^i$ of $G$ for each term $\pr(\y^i_{\x^i}, \w^i_{\v^i}|\e^i, \u)$ in $L$ (see \sectionref{sec:background}). Join $G^1, \ldots, G^n$ so that their unit variables $\U$ are shared. This leads to model \(G'\).

\item Add a node $H$ to \(G'\) as a parent of all outcome nodes $\Z = \{[\Y^i], [[\W^i]]\}^n_{i=1}$. Node $H$ has states \(h_1,\ldots,h_n\) and prior $\theta'(h_i) = w_i$. Each node $Z \in \Z$ now has parents $\Pa_Z \cup \{H\}$, where $\Pa_Z$ are the parents of 
$Z$ in \(G'\) before node $H$ is added. Let $z^i$ be the state of $Z$ in the corresponding instantiation $\y^i\w^i$ of objective function \(L\). The new CPT for $Z$ is: 
\bigskip
    \begin{center}
    \vspace{-5mm}
    \small
    \begin{tabular}{c|c|c|c}
        $\Pa_Z$ & $H$ & $Z$ & $\theta'(Z|\Pa_Z, H)$ \\
        \hline
        $\pa$ & $h_i$ & $z^i$ & $\theta(z^i|\pa)$   \\
        $\pa$ & $h_i$ & $\bar{z}^i$ & $\theta(\bar{z}^i|\pa)$   \\
        $\pa$ & $\bar{h}_i$ & $z^i$ & $1.0$  \\
        $\pa$ & $\bar{h}_i$ & $\bar{z}^i$ & $0.0$  \\
    \end{tabular}
    \end{center}
Here, \(\bar{z}^i, \bar{h}_i\) denote any states of variables
\(Z,H\) that are distinct from states \(z^i,h_i\).

\end{enumerate}
\end{definition}

We say the objective model 
\(G'\) has {\em \(n\) components,}
and call \(H\) the 
{\em mixture variable} as it encodes a mixture 
of the objective function terms. 
The CPTs for variables \([\Y^i], [[\W^i]]\) in model \(G'\)
reduce to their original CPTs in SCM \(G\)
when $H = h_i$, and imply $[\Y^i]=\y^i, [[\W^i]]=\w^i$ when $H \neq h_i.$
The objective $L(\u)$ in SCM 
$G$ is a classical probability in the objective model $G'$ (proof
in \appendixref{app:reduction proof}).

\begin{theorem}
\label{thm:unit-selection-reduction}
Consider an SCM $G$ with unit variables $\U$. Let $L$ be the
objective function in \equationref{eqn:objective-function},
and let $G'$ be an objective model for $\langle G,L\rangle$. Let $\X=\{[\X^i]\}_{i=1}^n$,$\Y=\{[\Y^i]\}_{i=1}^n$, $\W=\{[[\W^i]]\}_{i=1}^n$, $\V=\{[[\V^i]]\}_{i=1}^n$ and $\E=\{\E^i\}_{i=1}^n$. We have $L(\u) = \pr'(\y, \w|\x, \v, \e, \u),$ 
where $\y,\w,\x,\v,\e$ are the instantiations of variables \(\Y,\W,\X,\V,\E\) in objective function \(L\).
\end{theorem}

\begin{figure}[t]
\centering
{\includegraphics[width=0.90\textwidth]{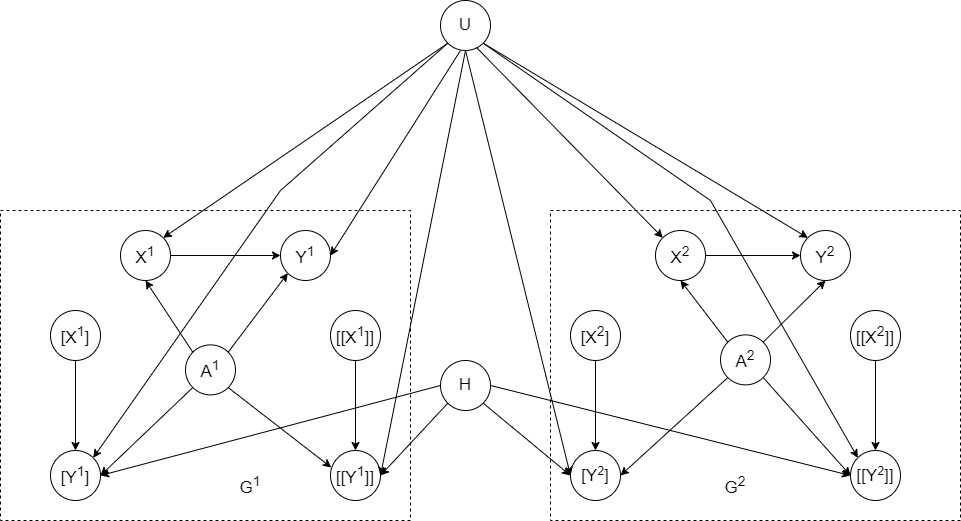}}
{\caption{An objective model with two components for the SCM in \figureref{fig:base-triplet}.
\label{fig:objective model}}}
\end{figure}

Consider the SCM in \figureref{fig:base-triplet}
and the causal objective function 
$L(u) = w_1 \cdot \pr(y_x, y'_{x'}|u) + w_2\cdot\pr(y_x, y_{x'}|u)$. 
Figure~2
shows a corresponding
objective model \(G'\) constructed according to
Definition~\ref{def:objective model}. We now have
$L(u)=\pr'([y^1], [[{y'}^1]], [y^2], [[y^2]] \mid [x^1], [[{x'}^1]], [x^2], [[{x'}^2]],u).$

Theorem \ref{thm:unit-selection-reduction} suggests that we can optimize the objective function $L(\u)$ on an SCM \(G\) by 
computing the instantiation $\argmax_\u \pr(\y,\w|\x,\v,\e,\u)$ 
on an objective model $G'$. 
The is similar to the classical MAP problem on model $G'$, 
except that 
the optimized variables $\U$ appear after the conditioning 
operator instead of before it. 
This leads to our definition of 
the Reverse-MAP problem.

\begin{definition}[Reverse-MAP]
\label{def:reverse-MAP}
Consider an SCM \(G\) with distribution \(\pr\) and
suppose \(\U,\E_1,\E_2\) are disjoint sets of variables in \(G\).
The Reverse-MAP instantiation for variables $\U$ and instantiations \(\e_1,\e_2\) is defined as follows:
\(
\rmap(\U, \e_1, \e_2) \triangleq \argmax_{\u} \pr(\e_1 \mid \u, \e_2).\)
\end{definition}

To see the connection between Reverse-MAP and MAP,
note that 
$\argmax_{\u}\pr(\e_1|\u, \e_2) = \argmax_{\u}\pr(\u, \e_1, \e_2) / \pr(\u, \e_2)$ where 
\(\argmax_{\u}\pr(\u, \e_1, \e_2)
= \argmax_{\u}\pr(\u | \e_1, \e_2)\) is the known
MAP problem~\citep{pearlBNbook}.
In general, the MAP instantiation $ \argmax_{\u}\pr(\u,\e_1, \e_2)$ is not the Reverse-MAP instantiation since $\pr(\u, \e_2)$ also depends on $\U$; see 
\appendixref{app:example-map-rmap} for a concrete example
that illustrates this point.
We now have the following result,
proven in \appendixref{app:poly reductions proof}.

\begin{corollary}
\label{cor:unit-selection-rmap}
There are polynomial-time reductions between 
the Reverse-MAP problem and the unit selection problem 
with objective functions in the form of \equationref{eqn:objective-function}.
\end{corollary}


\label{sec:rmap-complexity}
We next characterize the complexity of Reverse-MAP under
different conditions. 
Consider a decision version of the problem, D-Reverse-MAP,
defined as follows. 

\begin{definition}[D-Reverse-MAP]
\label{def:D-Reverse-MAP}
Given an SCM with rational parameters that induces distribution \(\pr\), 
some target variables $\U$, 
some evidence $\e_1, \e_2$ and a rational threshold $p$,
the D-Reverse-MAP problems asks whether there is
an instantiation $\u$ of $\U$ such that $\pr(\e_1|\u, \e_2) > p$.
\end{definition}

The next theorem shows that 
D-Reverse-MAP is $\NP^\PP$-complete, 
like classical MAP \citep{jair/ParkD04}.
Its proof can be found in 
\appendixref{app:d-reverse-map proof NP-PP}.

\begin{theorem}
\label{thm:D-Reverse-MAP complexity}
D-Reverse-MAP is $\NP^\PP$-complete.
\end{theorem}

We can now characterize the complexity of unit selection using
\theoremref{thm:D-Reverse-MAP complexity} and \corollaryref{cor:unit-selection-rmap}.

\begin{corollary}
Unit selection is $\NP^\PP$-complete assuming the objective
function in \equationref{eqn:objective-function}.
\end{corollary}

In an SCM, exogenous (root) variables represent all uncertainties in the model and the endogenous (internal) variables are uniquely determined by exogenous variables. This property of SCMs
significantly reduces the complexity of unit selection
when the unit variables correspond to all SCM exogenous variables.
This is implied by the following result which is proven
in \appendixref{app:d-reverse-map proof NP}.

\begin{theorem}
\label{thm:d-reverse-MAP NP}
D-Reverse-MAP is $\NP$-complete if its target variables 
are all the SCM root variables.
\end{theorem}

\begin{corollary}
Unit selection is $\NP$-complete when the unit variables
are all the SCM exogenous (root) variables, assuming the
objective functions in \equationref{eqn:objective-function}.
\end{corollary}


\section{Unit Selection using Variable Elimination}
\label{sec:VE-RMAP}

\sectionref{sec:unit-selection-reduction} provided
a reduction from unit selection
on an SCM to Reverse-MAP on an objective model.
In \sectionref{sec:ve}, we provide a 
variable elimination (VE) algorithm
for Reverse-MAP which can be applied to the objective
model to solve unit selection. 
In \sectionref{sec:treewidth bounds}, we analyze
the complexity of this method and compare it to the complexity of Reverse-MAP on the underlying SCM.

\subsection{Reverse-MAP using Variable Elimination}\label{sec:ve-algorithm}
\label{sec:ve}
\def\GG{{\cal G}}

Our VE algorithm for Reverse-MAP will employ the same
machinery and techniques used in the VE algorithm for classical MAP~\citep{ai/Dechter99}. 
Hence, we will first review the VE algorithm for MAP using
the treatment in \citep[Ch~10]{DarwicheBook09}
and then discuss the algorithm for Reverse-MAP.

The VE algorithm is based on the notion of a {\em factor}
\(f(\X)\) which maps each instantiation \(\x\) of variables \(\X\) 
into a non-negative number \(f(\x)\). 
VE employs a number of factor operations including multiplying two factors (\(f \cdot g\)),
summing out a variable from a factor (\(\sum_X f\)),
maximizing out a variable from a factor (\(\max_X f\)),
and dividing two factors (\(f / g\)). 
Let $G$ be an SCM and assume its variables $\Z$ are partitioned into three disjoint sets $\U, \V, \E$, 
where $\U$ are the {\em target variables} and $\E$ 
are the {\em evidence variables.} 
Let $\S = \Z \setminus \U$ in the following
discussion. We will treat the CPT of each variable $Z$ in SCM \(G\)
as a factor over $Z$ and its parents $\P$, denoted $f_Z(Z\P)$.
The SCM distribution is then \(\pr(\Z) = \prod_{Z \in \Z} f_Z\).
We capture evidence \(\e\) by creating an evidence factor $\lambda_e(E)$ for each \(e \in \e\) with $\lambda_e(e') = 1$ if
\(e' = e\) and $\lambda_e(e') = 0$ otherwise.
The {\em MAP probability} is then given 
by\footnote{\label{foot:scalar}The left side of Equation~\ref{eq:map} is
a scalar (probability) while the right side is a factor
over an empty set of variables, which is called a 
{\em scalar factor.} Such a factor maps only one instantiation,
the empty one, to a scalar.}
\begin{equation}
\label{eq:map}
\map_p(\U,\e) =
    \max_\u \pr(\u, \e) = \max_\u \sum_\v \pr(\u,\v,\e) = \max_{\U} \sum_{\S} \prod_{Z \in \Z} f_{Z} \prod_{e \in \e} \lambda_e(E)
\end{equation}

\begin{wrapfigure}[11]{r}{0.23\textwidth}
  \begin{center}
  \vspace{-10mm}
    \includegraphics[width=0.16\textwidth]{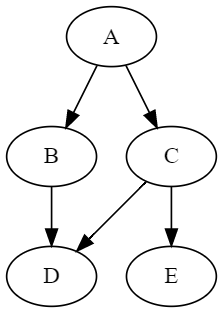}
  \end{center}
  \vspace{-7mm}
  \caption{SCM \label{fig:ve-scm}}
\end{wrapfigure}
This is in contrast to the {\em MAP instantiation}
which is \(\argmax_\u \pr(\u, \e)\). With some minor
bookkeeping, the VE algorithm
for computing the MAP probability can also return a
MAP instantiation; see, 
e.g.,~\citep[Ch~10]{DarwicheBook09}. Hence,
we will focus next on computing the MAP probability.

Consider the SCM in \figureref{fig:ve-scm} and
suppose $\U = \{A, B\}$ and the evidence $\e$ is $\{E = e\}$. 
In this case, \(\map_p(AB,e)\) will be equal to
\begin{equation}
    \max_{AB} \sum_{CDE} f_A(A) f_B(AB) f_C(AC) f_D(BCD) f_E(CE) \lambda_e(E) 
    \label{eqn:map-expression}
\end{equation}

\begin{algorithm}[t]
\small
    \caption{MAP\_VE$(G, \U, \e)$}\label{alg:ve-map}
    \Ainput{SCM $G$, target variables $\U$, evidence $\e$ } \\
    \Aoutput{scalar factor containing $\map_p(\U,\e)$}
    \begin{algorithmic}[1]
    \Procedure{main}{}
        \State $\pi_\S \gets$ an elimination order for non-target variables
        \State $\pi_\U \gets$ an elimination order for target variables $\U$
        \State $\F \gets \{f: f \mbox{ is a CPT of SCM } G\} \cup \{\lambda_e: \lambda_e \mbox{ is an evidence factor for } e\in \e\}$ \label{line:joint}
        \State $\GG \gets$ \Call{eliminate}{$\sum,\F, \pi_\S$} \label{line:sum}
        \State $p \gets$ \Call{eliminate}{$\max,\GG, \pi_\U$}  \label{line:max}
        \State \Return $p$
    \EndProcedure 
    \end{algorithmic}
\end{algorithm}

\begin{algorithm}[t]
\small
    \Ainput{an operation \(\bigcirc \in \{\sum,\max\}\), a set of factors $\F$, 
    a total variable order $\pi$} \\
    \Aoutput{a set of factors}
    \caption{Eliminating Variables using Sum or Max
    \label{alg:eliminate}}
    \begin{algorithmic}[1]
 \Procedure{eliminate}{$\bigcirc,\F,\pi$} 
    \For{$i = 1$ to length of order $\pi$}
        \State $V \gets i^{th}$ variable in order \(\pi\)
        \State $\GG \gets $ factors in $\F$ that mention variable $V$ \label{line:V-factors}
        \State $f_i \gets \prod_{f \in \GG} f$  
        \State $f_i \gets \bigcirc_{V} f_i$ \label{line:new-factor}
        \State replace factors \(\cal G\) in \(\F\) with factor \(f_i\)
        \EndFor
        \State \Return{$\F$}
    \EndProcedure 
\end{algorithmic}
\end{algorithm}

A naive evaluation of this expression multiplies
all factors to yield a factor $f(ABCDE)$ over all variables, then computes $\max_{AB}\sum_{CDE} f(ABCDE)$, leading to $O(n\exp(n))$ complexity where $n$ is the number of model variables. The VE algorithm tries to compute this expression more efficiently with pseudocode provided in \algorithmref{alg:ve-map} (MAP\_VE). The product of factors \(\F\) on Line~\ref{line:joint} represents the joint distribution $\pr(\Z, \e)$ so
we first sum out variables $\S$ from \(\F\) on
Line~\ref{line:sum} to compute a set of factors \(\GG\) whose product represents the marginal $\pr(\U,\e)$. 
We then maximize out variables $\U$ from \(\cal G\) on
Line~\ref{line:max} leading to a scalar factor \(p\) that 
contains the MAP probability (see Footnote~\ref{foot:scalar}).
\algorithmref{alg:ve-map} eliminates variables one by one
using \algorithmref{alg:eliminate} and a total variable order $\pi = \langle \pi_\S, \pi_\U \rangle$, known as an {\em elimination order.} MAP\_VE requires
variables $\U$ to appear last in order $\pi$ since summation does not commute with maximization. An order that satisfies this constraint is known as a {\em $\U$-constrained elimination order}. The complexity of 
MAP\_VE depends on the used elimination order $\pi$. In each elimination step of \algorithmref{alg:eliminate}, we multiply all factors that mention variable $\pi(i)$ to obtain factor $f_i$ on Line~\ref{line:new-factor}. The 
variables in $f_i$ are called a cluster $\C_i$ so 
eliminating variables $\pi(1), \ldots, \pi(n)$ induces clusters $\C_1, \ldots, \C_n.$ The {\em width} $w$ of elimination order \(\pi\) is the size of largest cluster minus one and the complexity of MAP\_VE is $O(n\exp(w))$.
 
The table below depicts the trace of MAP\_VE 
when computing the MAP probability in \equationref{eqn:map-expression}
using the elimination order $\pi = E,D,C,B,A$.
The trace shows that MAP\_VE evaluates the 
following factorized expression and that the width of
order $\pi$ is $2$ (largest cluster has size \(3\)):
\[
    \map_p(AB,e) = \max_A f_A(A) \bracket{\max_B f_B(AB) \bracket{\sum_C f_C(AC) \bracket{\sum_D f_D(BCD) \bracket{\sum_E f_E(CE) \lambda_E}}}}
\]

\begin{center}
\begin{small}
    \begin{tabular}{ |l|c|l|l|l| } 
        \hline
         $i$ & eliminated var & factors $\cal G$ (Line~\ref{line:V-factors}, \algorithmref{alg:eliminate}) & new factor $f_i$ (Line~\ref{line:new-factor}, \algorithmref{alg:eliminate}) & $\C_i$   \\
         \hline
         $1$ & $E$ & $f_E(CE) \ \lambda_E$ & 
         $f_1 = \sum_E f_E(CE) \ \lambda_E$ & CE \\ 
         $2$ & $D$ & $f_D(BCD)$ & $f_2 = \sum_D f_D(BCD)$ & BCD \\ 
         $3$ & $C$ & $f_C(AC)\ f_1(C)\ f_2(BC)$ & $f_3 = \sum_C f_C(AC)\ f_1(C)\ f_2(BC)$ & ABC \\ 
         $4$ & $B$ &   $f_B(AB)\ f_3(AB)$  & $ f_4 = \max_B f_B(AB)\ f_3(AB)$ & AB\\
         $5$ & $A$ & $f_{A}(A) \ f_4(A)$ & $p = \max_A f_{A}(A)\ f_4(A)$ & A \\
         \hline
    \end{tabular}
\end{small}
\end{center}
\bigskip

Choosing a good elimination order is critical for the 
complexity of VE. The {\em treewidth} of an SCM $G$ is defined as the minimum width attained by any elimination order. Since MAP requires $\U$-constrained orders, 
the {\em \(\U\)-constrained treewidth} of $G$ is defined as the minimum width attained by any $\U$-constrained elimination order~\citep{jair/ParkD04}.

We are now ready to introduce our VE algorithm for Reverse-MAP. Again, we assume that the model
variables $\Z$ are partitioned into disjoint 
sets $\U, \V, \E$, where $\U$ are the target variables
and \(\S = \Z \setminus \U\). 
But we further partition the evidence variables $\E$ into
\(\E_1\) and \(\E_2\). Again, we focus on computing the
Reverse-MAP probability $\rmap_p(\U,\e_1,\e_2)$ 
instead of the instantiation:
\[
\max_\u \pr(\e_1|\u, \e_2) 
= \frac{\pr(\u, \e_1, \e_2)}{\pr(\u, \e_2)} 
= \frac{\sum_\v \pr(\u, \v, \e_1, \e_2)}{\sum_\v \pr(\u, \v, \e_2)} 
= \max_{\U} \frac{\displaystyle \sum_{\S} \prod_{Z \in \Z} f_{Z} \prod_{e \in \e_1\cup\e_2} \lambda_e}
     {\displaystyle \sum_{\S} \prod_{Z \in \Z} f_{Z} \prod_{e \in \e_2} \lambda_e}
\]

Our algorithm, called RMAP\_VE, runs two passes of elimination as shown in \algorithmref{alg:ve-rmap}. In the first pass (Line \ref{line:sumout-1}), we sum out variables $\S$ under evidence $\e_1,\e_2$ and in the second pass (Line~\ref{line:sumout-2}), we sum out variables $\S$ under evidence $\e_2$. This leads to two sets of factors $\GG_1$ and $\GG_2$
which correspond to marginal distributions $\pr(\U,\e_1,\e_2)$ and $\pr(\U, \e_2)$. 
Now we need to divide
$\pr(\U,\e_1,\e_2)$ and $\pr(\U, \e_2)$ to compute 
$\pr(\e_1|\U, \e_2)$. 
We next show that this can be done efficiently by 
``dividing'' $\GG_1$ and $\GG_2$ as shown on
Line~\ref{line:division}.
The key idea is that if we run the two passes of elimination according to the same elimination order, then there will be a one-to-one correspondence between the factors in $\GG_1$ and $\GG_2$. Let \((g_1^i,g_2^i)\)
be the corresponding pairs of factors
for \(i = 1, \ldots, k\) where \(k=|\GG_1|=|\GG_2|\). 
What we need is \(\left(\prod_{i=1}^n g_1^i\right) / \left(\prod_{i=1}^n g_2^i\right)\) since this represents
$\pr(\e_1|\U, \e_2)$. But due to the mentioned
correspondence, this equals 
\(\prod_{i=1}^n g_1^i/g_2^i\).
Thus, we can divide each pair of corresponding factors to obtain the set of factors \(\GG\) 
as done on Line~ \ref{line:division}. 
We finally maximize out target variables $\U$ from $\GG$ to obtain the Reverse-MAP probability (Line~\ref{line:final}).

\begin{algorithm}[ht]
\small
    \caption{RMAP\_VE($G, \U, \e_1, \e_2)$}\label{alg:ve-rmap}
    \Ainput{SCM $G$, target variables $\U$, evidence $\e_1$ and $\e_2$} \\
    \Aoutput{scalar factor containing $\rmap_p(\U,\e_1,\e_2)$}
    \begin{algorithmic}[1]
    \Procedure{main}{}
        \State $\pi_\S \gets$ an elimination order for non-target variables
        \State $\pi_\U \gets$ an elimination order for target variables $\U$
        \State $\F_1 \gets \{f: f \mbox{ is a CPT of SCM } G\} \cup \{\lambda_e: \lambda_e \mbox{ is an evidence factor for } e\in \e_1,\e_2\}$
        \label{line:sumout-1}
        \State $\F_2 \gets \{f: f \mbox{ is a CPT of SCM } G\} \cup \{\lambda_e: \lambda_e \mbox{ is an evidence factor for } e\in \e_2\}$
        \label{line:sumout-2}
        \State $\GG_1 \gets$ \Call{eliminate}{$\sum,\F_1, \pi_\S$}
        \State $\GG_2 \gets$ \Call{eliminate}{$\sum,\F_2, \pi_\S$}
        \State $\GG \gets \{g_1/g_2: $ $g_1, g_2$ are corresponding factors in $\GG_1, \GG_2\}$ \label{line:division}
        \State $p \gets $ \Call{eliminate}{$\max,\GG, \pi_\U$} \label{line:final}
        \State \Return $p$
    \EndProcedure 
    \end{algorithmic}
\end{algorithm}

RMAP\_VE has the same complexity as MAP\_VE if both
use the same elimination order. Suppose there are $k$ factors in $\GG_1/\GG_2/\GG$ and the largest factor has size $c$. The cost of division on Line~\ref{line:division} is $O(k\exp(c))$ while the cost of maximization on Line~\ref{line:final} is at least $O(k\exp(c))$
so the cost of division is dominated by the cost of maximization. Hence, the complexity of RMAP\_VE is still $O(n\exp(w))$ where \(n\) is the number of variables
and $w$ is the width of used $\U$-constrained order $\pi$.

%

\subsection{Bounding the Complexity of Unit Selection using Variable Elimination}
\label{sec:treewidth bounds}
We can solve unit selection by applying RMAP\_VE
to an objective model of the SCM as shown by 
\theoremref{thm:unit-selection-reduction}.
However, RMAP\_VE (and MAP\_VE) is expected to be more expensive on the objective model compared to the 
given SCM since the former is larger and denser 
than the latter. But how much more expensive?
In particular, is RMAP\_VE always tractable on the objective model when it is tractable on the underlying SCM? 
We consider this question next using the lens of treewidth
which is commonly used to analyze elimination algorithms.
Recall also that MAP\_VE and RMAP\_VE have 
the same complexity when applied to the same SCM using 
the same target variables.

Our starting point is to study the treewidth of an 
objective model in relation to the treewidth of its 
underlying SCM. We will base
our study on the techniques and results reported in
\citep{han2022on} which studied the complexity of counterfactual reasoning. In particular, given an elimination 
order $\pi$ of SCM $G$, 
we next show how to construct an elimination 
order $\pi'$ for the 
objective model $G'$ while providing a bound on the width of
order \(\pi'\) in terms of the width of order \(\pi\).
Recall that we use $[X]$ and $[[X]]$ to denote the copies 
of variable $X$ in a triplet model where 
$X = [X] = [[X]]$ if $X$ is exogenous.
Moreover, if $U$ is a unit variable, 
then $U = U^1 = \cdots = U^n$ in an objective model.

\begin{definition}
\label{def:elimination order}
Let $G$ be an SCM and \(G'\) be a corresponding
objective model with \(n\) components.
If $\pi$ is an elimination order for \(G\),
the corresponding elimination order $\pi'$ for $G'$ is obtained by
replacing each non-unit variable $X$ in $\pi$ by  $X^1, \ldots, X^n, [X^1], \ldots, [X^n], [[X^1]], \ldots, [[X^n]]$
then
appending the mixture variable $H$ to the end of the order.
\end{definition}

Consider the elimination order $\pi = A, X, Y, U$ for the SCM in \figureref{fig:base-triplet}. The corresponding elimination order $\pi'$ for the objective model in \figureref{fig:objective model} 
is as follows:
\begin{equation*}
    \pi' = A^1, A^2, X^1, X^2, [X^1], [X^2], [[X^1]], [[X^2]], Y_1, Y_2, [Y_1], [Y_2], [[Y_1]], [[Y_2]], U, H
    \label{eqn:n-twin-order}
\end{equation*}

The following bound (\theoremref{thm:objective-width}) follows from \lemmaref{lem:add-node} and Theorem~\ref{thm:n-world} 
which concerns {\em \(n\)-world models.}
Given an SCM \(G\) and a subset \(\U\) of its roots, an \(n\)-world model
is obtained by creating \(n\) copies of \(G\) that
share nodes \(\U\)~\citep{han2022on}. This notion
corresponds to parallel worlds models~\citep{ijcai/AvinSP05}
when \(\U\) contains all roots of SCM \(G\).
An objective model with \(n\) components can be viewed as a 
\(3n\)-world model but with an additional mixture 
node \(H\) and some edges that originate from \(H\).
\lemmaref{lem:add-node} and \theoremref{thm:objective-width} are proven
in \appendixref{app:lemma-add-node} and
\appendixref{app:objective width}.

\begin{lemma}
\label{lem:one node}
\label{lem:add-node}
Consider an SCM $G$ and suppose SCM $G'$ is obtained from $G$ by adding a root node $H$ as a parent of some nodes in $G$. 
If $\pi$ is an elimination order for $G$ and has width $w$,
then $\pi' = \langle \pi, H \rangle$ is an elimination order 
for $G'$ and has width $w' \leq w + 1$.
\end{lemma}

\begin{theorem}[\cite{han2022on}]
\label{thm:n-world}
Consider an SCM \(G\), a subset \(\U\) of its roots and
a corresponding \(n\)-world model \(G'\).
If \(G\) has 
an elimination order \(\pi\) with width \(w\), then there exists a corresponding elimination order \(\pi'\) of \(G'\) 
that has width $w' \leq n(w+1) - 1$.
\end{theorem}

\begin{theorem}
\label{thm:width}
\label{thm:objective-width}
Consider an SCM $G$ and a corresponding objective model
\(G'\) with \(n\) components.
Let \(\pi\) be an elimination order for \(G\) and let 
\(\pi'\) be the corresponding elimination order for \(G'\). 
If \(\pi\) has width \(w\) and \(\pi'\) has width \(w'\),
then $w' \leq 3n(w + 1)$.
\end{theorem}

\begin{corollary}
If $w$ is the treewidth of an SCM $G$ and $w'$ is the treewidth of a corresponding objective model $G'$ with
\(n\) components, then  $w'\leq 3n(w+1)$.
\label{cor:objective-treewidth}
\end{corollary}

As mentioned earlier, RMAP\_VE and MAP\_VE require
a \(\U\)-constrained elimination orders
in which unit variables $\U$ appear last in the order. Hence,
a \(\U\)-constrained elimination order for an objective 
model must place the mixture variable $H$ before $\U$.
This leads to the next definition.

\begin{definition}
\label{def:constrain-elimination-order}
Let $G$ be an SCM with unit variables $\U$ and let
\(G'\) be a corresponding objective model with \(n\)
components. If $\pi$ is a $\U$-constrained elimination 
order for $G$, the corresponding $\U$-constrained elimination order $\pi'$ for \(G'\) is obtained by
replacing each non-unit variable $X$ in $\pi$ by  $X^1, \ldots, X^n, [X^1], \ldots, [X^n], [[X^1]], \ldots, [[X^n]]$
then
inserting mixture variable $H$ just before $\U$.
\end{definition}

Consider the \(\U\)-constrained order 
\(\pi = A, X, Y, U\) for the SCM
in \figureref{fig:base-triplet}.
The corresponding $\U$-constrained elimination order for the objective model in 
\figureref{fig:objective model} 
is
\begin{equation*}
     \pi' = A^1, A^2, X^1, X^2, [X^1], [X^2], [[X^1]], [[X^2]], Y_1, Y_2, [Y_1], [Y_2], [[Y_1]], [[Y_2]], H, U
    \label{eqn:constrain-n-twin-order}
\end{equation*}

We now have the following bound on the \(\U\)-constrained treewidth
of objective models, which is somewhat unexpected when compared
to the bound on treewidth. In particular, 
while the bound on treewidth grows linearly 
in the number of components in the objective model, the bound on
\(\U\)-constrained treewidth is independent of such a number.
Moreover, the bound on \(\U\)-constrained treewidth 
can depend on 
the number of unit variables which is not the case for treewidth.

\begin{theorem}
\label{thm:constrained width bound}
\label{thm:objective-constrain-width}
Let $G$ be an SCM with unit variables $\U$ and let
\(G'\) be a corresponding objective model. 
If $\pi$ is a $\U$-constrained elimination 
order for $G$ with width \(w\) and \(\pi'\) is
the corresponding $\U$-constrained elimination order 
for \(G'\) with width \(w'\), then
$w' \leq \max(3w + 3, |\U|)$.
If the objective function in
\equationref{eqn:objective-function} has one outcome
variable (\(\Y^i=\W^i=\{Y\}\) for all \(i\)), 
then $w' \leq 3w + 3$.
\end{theorem}

\begin{corollary}
\label{cor:objective-constrained-treewidth}
Let \(G\) be an SCM with unit variables \(\U\) and 
let \(G'\) be a corresponding objective model.
If $w$ and $w'$ are the $\U$-constrained treewidths of \(G\) 
and \(G'\), then $w' \leq \max(3w + 3, |\U|)$.
Moreover, if the objective function in
\equationref{eqn:objective-function} has a single outcome
variable, then $w' \leq 3w + 3$.
\end{corollary}
The above bounds can be significantly tighter depending 
on the objective function properties. 
\corollaryref{cor:objective-constrained-treewidth} 
identifies one such property which is satisfied by the benefit function in~\citep{ijcai/LiP19}; see Equation~\eqref{eqn:benefit-ang}. Moreover, 
the factor \(3\) in these bounds is an implication of using 
a triplet model which may not be necessary. Consider 
components 
\(\pr(\y^i_{\x^i},\w^i_{\v^i}|\e^i,\u)\) in the objective function of \equationref{eqn:objective-function}. 
If \(\E^i\!\!=\!\!\emptyset\) for all \(i\), 
then a twin model is
sufficient when building an objective model (similarly
if \(\Y^i\!\!=\!\!\X^i\!\!=\!\!\emptyset\) or \(\W^i\!\!=\!\!\V^i\!\!=\!\!\emptyset\)). 
The objective function in Equation~\eqref{eqn:benefit-ang}, from~\citep{ijcai/LiP19}, has
\(\E^i\!\!=\!\!\emptyset\) for all \(i\)
so it leads to the tighter bound \(w' \leq 2w + 2\). 
More generally, if the objective function 
properties lead to removing the
dependence on \(|\U|\) in the bound of 
\corollaryref{cor:objective-constrained-treewidth}, then
RMAP\_VE on an objective
model is tractable if RMAP\_VE (MAP\_VE)
is tractable on the underlying SCM.
Otherwise, the bound in 
\corollaryref{cor:objective-constrained-treewidth}
does not guarantee this. 
Recall that MAP, Reverse-MAP
and unit selection using~\equationref{eqn:objective-function}
are all \(\NP^\PP\)-complete as shown earlier.

We provide in Appendix~\ref{app:experiment} a preliminary experiment and an extensive discussion in relation to the complexities of three algorithms:  
(1)~MAP\_VE (Algorithm~\ref{alg:ve-map}) which solves MAP by operating on an SCM;
(2)~RMAP\_VE (Algorithm~\ref{alg:ve-rmap}) which solves unit selection by operating on an objective model;
and (3)~a baseline, bruteforce method which solves unit selection by operating on a twin or triplet model (depending on the objective function).
The main finding of the experiment is that, as the size of the problem grows,\footnote{The size of the problem is measured by the number of nodes in the SCM and the number of unit variables.}
the gap between the complexities of MAP\_VE and RMAP\_VE narrows while the gap between the complexities of  RMAP\_VE and the bruteforce method grows (the bruteforce method is significantly worse and becomes impractical pretty quickly).

Appendix~\ref{app:experiment} also identifies a class of SCM structures (and unit variables) for which the number of unit variables is unbounded but the complexity of RMAP\_VE on an objective model is bounded.

\section{Conclusion}
\label{sec:conclusion}

We studied the unit selection problem in a computational
setting which complements existing studies. We assumed a
fully specified structural causal model so we can compute
point values of causal objective functions, allowing us
to entertain a broader class of functions than
is normally considered. We showed that the unit selection
problem with this class of objective functions is 
\(\NP^\PP\)-complete, similar to the classical MAP problem,
and identified an intuitive condition under which it
is \(\NP\)-complete. We further provided an exact algorithm
for the unit selection problem 
based on variable elimination and characterized its 
complexity in terms of treewidth, while relating this
complexity to that of MAP inference. In the process,
we defined a new inference problem,
Reverse-MAP, which is also \(\NP^\PP\)-complete but
captures the essence of unit selection more than MAP
does.

\acks{We thank Yizuo Chen, Yunqiu Han, 
Ang Li and Scott Mueller for providing useful
feedback on an earlier version of this paper.
This work has been partially supported by ONR grant N000142212501.}

\bibliography{bib/reference,bib/references,bib/references2,bib/refs}

\newpage
\appendix

\section{Proof of \theoremref{thm:unit-selection-reduction}}
\label{app:reduction proof}
The proof of this theorem requires a lemma which requires
the following definition.
We will say that a set of variables $\Z$ {\em decomposes} a DAG if removing the outgoing edges from $\Z$ splits the DAG into at least
two disconnected components. 

\begin{lemma}
Consider an SCM $G$ with distribution $\pr$ and three disjoint set of variables $\X, \Y, \Z$. Suppose $\Z$ decomposes $G$
into disconnected components $G_1$ and $G_2$. If $\X_1, \Y_1$ are subsets of $\X, \Y$ pertaining to $G_1$, 
and $\X_2, \Y_2$ are subsets of $\X, \Y$ pertaining to  $G_2$,
then $\pr(\y|\x, \z) = \pr(\y_1|\x_1, \z) \pr(\y_2|\x_2, \z)$.
\label{lem:disconnect}
\end{lemma}

\begin{proof}
Since \(\Z\) decomposes \(G\), we have $\dsep_G(\X_1, \Z, \X_2)$ and $\dsep_G(\X_1\Y_1, \Z, \X_2\Y_2)$. We have:
    \begin{align}
        \pr(\y|\x, \z) &= \frac{\pr(\y, \x|\z)}{\pr(\x|\z)} = \frac{\pr(\y_1, \y_2, \x_1, \x_2|\z)}{\pr(\x_1, \x_2|\z)} = \frac{\pr(\y_1, \x_1|\z) \pr(\y_1, \x_1|\z)}{\pr(\x_1|\z) \pr(\x_2|\z)} \label{eqn:cutset} \\
                       &= \left[\frac{\pr(\y_1, \x_1|\z)}{\pr(\x_1|\z)}\right] \left[\frac{\pr(\y_2, \x_2|\z)}{\pr(\x_2|\z)}\right] 
                       = \pr(\y_1|\x_1, \z) \pr(\y_2|\x_2, \z) \nonumber
    \end{align}
\equationref{eqn:cutset} follows from $\dsep_G(\X_1, \Z, \X_2)$ and $\dsep_G(\X_1\Y_1, \Z, \X_2\Y_2)$. This concludes our proof. Although we only consider the case of two subnetworks here, it is easy to see that this lemma generalizes to an arbitrary number of subnetworks decomposed by $\Z$. 
\end{proof}

We are now ready to prove \theoremref{thm:unit-selection-reduction}.
By construction of $G'$, $\U\cup \{H\}$ decomposes $G'$ into 
its $n$ components $G^1, G^2, \ldots G^n$. We have:
\begin{align}
    \pr'(\y, \w \mid \x, \v, \e, \u) &= \sum_{i=1}^n \pr'(\y, \w, h_i \mid \x, \v, \e, \u) \nonumber \\
        &= \sum_{i=1}^n \pr'(\y, \w \mid \x, \v, \e, \u, h_i) \pr'(h_i \mid  \x, \v, \e, \u) \nonumber \\
        &= \sum_{i=1}^n \pr'(\y, \w \mid \x, \v, \e, \u, h_i) \pr'(h_i) \label{eqn:h-prob} \\
        &= \sum_{i=1}^n \left[\prod_{j=1}^n \pr'(\y^j, \w^j  \mid \x^j, \v^j, \e^j, \u, h_i) \right] \pr'(h_i)  \label{eqn:call-lemma}  \\
        &= \sum_{i=1}^n \left[\prod_{j\neq i} pr'(\y^j, \w^j  \mid \x^j, \v^j, \e^j, \u, h_i) \right] \pr'(\y^i, \w^i  \mid \x^i, \v^i, \e^i, \u, h_i) \pr'(h_i) \label{eqn:cpt} \\
        &= \sum_{i=1}^n \left[\prod_{j\neq i } 1.0 \right] \pr'(\y^i, \w^i  \mid \x^i, \v^i, \e^i, \u) \: w_i \label{eqn:triplet} \\
        &= \sum_{i=1}^n  w_i \: \pr(\y^i_{\x^i}, \w^i_{\v^i}  \mid \e^i, \u) \nonumber \\
        &= L(\u) 
\end{align}
\equationref{eqn:h-prob} follows since the auxiliary root $H$ is d-separated from $\X \cup \V \cup \E \cup \U$.
\equationref{eqn:call-lemma} follows from Lemma \ref{lem:disconnect}  since $\U \cup \{H\}$ decomposes $G$ such that all triplet  models are disconnected.  \equationref{eqn:cpt} follows from the construction of the new CPTs of $\Y^i$ and $\W^i$:  if $H = h_i$, then the original CPTs of $\Y^i$ and $\W^i$ are preserved, and the values of $\Y^j$ and $\W^j$ are fixed to $\y^j$ and $\w^j$ for all $j \neq i$. \equationref{eqn:triplet} follow from the property of the triplet network.

\section{Example for MAP and Reverse-MAP}
\label{app:example-map-rmap}
Consider the simple model in \figureref{fig:two-node}.
We have $\argmax_u \pr(u,v_1) = u_2$ for MAP while $\argmax_u \pr(v_1|u) = u_1$ for R-MAP.

\begin{figure}[h]
\floatconts
{fig:two-node}
{\caption{An example for illustrating the difference between MAP and R-MAP.}}
{
\subfigure{
    \includegraphics[width=3cm]{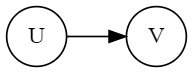}
}\qquad 
\subfigure{
    \begin{tabular}{c|c}
          $U$ & $\theta(A)$ \\
          \hline
          $u_1$ & 0.2 \\
          $u_2$ & 0.8
    \end{tabular}
}
\subfigure{
    \begin{tabular}{c|c|c}
          $U$ & $V$ & $\theta(V|U)$ \\
          \hline
          $u_1$ & $v_1$ & 0.6 \\
          $u_1$ & $v_2$ & 0.4 \\
          $u_2$ & $v_1$ & 0.3 \\
          $u_2$ & $v_2$ & 0.7
    \end{tabular}
}
}
\end{figure}

\section{Proof of \corollaryref{cor:unit-selection-rmap}}
\label{app:poly reductions proof}
We can reduce Reverse-MAP \(\argmax_{\u} \pr(\e_1|\u, \e_2)\)
to unit selection by choosing an objective 
function in the form of \equationref{eqn:objective-function} with the following
settings: \(n=1\), \(w_1=1\), \(\X^1 = \{\}\), 
\(\Y^1 = \E_1\) and \(\E^1 = \E_2\). This is clearly a polynomial-time
reduction. We already showed a reduction from unit selection
to Reverse-MAP in \theoremref{thm:unit-selection-reduction}.
Let \(|G|\) denote the size\footnote{The size of an SCM is the space needed to store the SCM structure and parameters. For example, if the SCM is represented by a functional Bayesian Network, its size is usually the total number of entries in the network CPTs.} of SCM \(G\) and \(n\) be
the number of components in the causal objective function.
By inspecting Definition~\ref{def:objective model}, 
we can immediately see that the time for constructing the objective model $G'$ is $O(n\cdot |G|)$. Moreover, the size
of objective model \(G'\) is also $O(n\cdot |G|)$.

\section{Proof of \theoremref{thm:D-Reverse-MAP complexity}}
\label{app:d-reverse-map proof NP-PP}
Membership in $\NP^\PP$ is immediate. Given an instantiation $\u$ of $\U$, it is easy to verify if $\u$ is a solution by querying the $\PP$-oracle if $\pr(\e_1|\u, \e_2) > p$ which is a problem known as D-MAR~\citep{DarwicheBook09}. To prove hardness, we show that E-MAJSAT~\citep{littman1998computational} can be reduced to D-Reverse-MAP in polynomial time, based on a slight modification of the reduction to classical MAP 
proposed in~\citep{park2002map,jair/ParkD04}.
The E-MAJSAT problem is defined as follows. Given a Boolean formula $\alpha$ over Boolean variables $\Z = \U \cup \V$: Is there an instantiation $\u$ of $\U$ such that the majority of instantiations $\v$ of $\V$
satisfy \(\u\v \models \alpha\) (formula $\alpha$ holds at \(\u\v\))?
We show that we can answer E-MAJSAT by answering
D-Reverse-MAP on an SCM $G_\alpha$ that simulates the formula
$\alpha$ and that can be constructed efficiently.
The SCM \(G_\alpha\) is constructed inductively, as shown in~\citep{jair/ParkD04},\footnote{\citep{jair/ParkD04} intended
to construct a Bayesian network, but their construction is an 
SCM since all internal nodes in the network have functional CPTs.}
and always has a single leaf node, denoted $S_\alpha.$
The construction is based on three rules:
(1)~If $\alpha=X$, then $G_\alpha$ has a 
single binary node $X$ with values \(\{0,1\}\) and a
uniform prior so $S_\alpha = X$;
(2)~If $\alpha = \neg \beta$, then $G_\alpha$ is constructed from $G_\beta$ by adding a binary node $S_\alpha$ as a child of $S_\beta \in G_\beta$ with structural equation $S_\alpha = 1 - S_\beta$; and
(3)~If $\alpha = \beta \land \gamma$ ($\alpha = \beta \lor \gamma$), then $G_\alpha$ is constructed from $G_\beta$ and $G_\gamma$ by adding a binary node $S_\alpha$ as a child of $S_\beta \in G_\beta$ and $S_\gamma \in G_\gamma$ with structural equation $S_\alpha = S_\beta \cdot S_\gamma$ ($S_\alpha = S_\beta + S_\gamma$). We are now ready for the last step of the proof. Given a Boolean formula $\alpha$ over
variables $\Z = \U \cup \V$, and given its SCM
\(G_\alpha\) that has distribution \(\pr\), 
we next show that there is an instantiation $\u$ 
such that 
$\pr(S_\alpha=1|\u) > 1/2$ (D-Reverse-MAP query) iff 
there is an instantiation $\u$ such that the majority of instantiations $\v$ of $\V$ satisfy $\u\v\models \alpha$ (E-MAJSAT query).
Let $|\Z| = n$ and $s_\alpha$ denote $S_\alpha=1$.
By construction of $G_\alpha$~\citep{jair/ParkD04}, we have
$\pr(\z) = 1/2^n$ for all instantiations \(\z\);
$\pr(s_\alpha|\z) = 1$ if $\z \models \alpha$ and $\pr(s_\alpha|\z) = 0$ otherwise. Then $\pr(\z, s_\alpha) = \pr(s_\alpha|\z)\pr(\z) = 1/2^n$ if $\z\models\alpha$ and 
$\pr(\z, s_\alpha) = 0$ otherwise. We finally have:
\[
\pr(s_\alpha|\u) 
=  \frac{\pr(\u, s_\alpha)}{\pr(\u)} 
= \frac{\sum_\v \pr(\u, \v, s_\alpha)}{\sum_\v \pr(\u, \v)} 
= \frac{\sum_{\v:\u\v \models \alpha} (1/2^n)}{\sum_{\v} (1/2^n)}
= \frac{\mathbf{card}(\{\v \in \mathcal{V}: \u\v \models \alpha\})}{\mathbf{card}(\mathcal{V})}
\]
Now that we have shown membership and hardness, 
D-Reverse-MAP is $\NP^\PP$-complete.

\section{Proof of \theoremref{thm:d-reverse-MAP NP}}
\label{app:d-reverse-map proof NP}
Recall \definitionref{def:D-Reverse-MAP} of D-Reverse-MAP:
Is there an instantiation \(\u\) such that
$\pr(\e_1|\u, \e_2) > p$?
Membership in $\NP$ is immediate. Since $\U$ is the set of exogenous variables, evidence variables $\E$ are functionally determined by $\U$. Hence, it is easy to check whether an instantiation $\u$ is a solution by first computing the instantiation \(\e\) of $\E$ implied by $\u$ (using structural equations) and then checking whether \(\e\) is consistent with $\e_1,\e_2$. If the answer is yes, then
$\pr(\e_1|\u, \e_2) = 1$, otherwise 
$\pr(\u, \e_2) = 0$ or $\pr(\e_1|\u, \e_2) = 0$.
To show hardness, we show that SAT can be reduced to D-Reverse-MAP under the conditions stated in the theorem. 
Given a Boolean formula $\alpha$ over variables $\U$, we construct an SCM $G_\alpha$ as in the proof of 
\theoremref{thm:D-Reverse-MAP complexity}.
SCM \(G_\alpha\) has a single leaf node $S_\alpha$ and its 
root nodes are $\U$. 
By construction of $G_\alpha$, we have $\pr(S_\alpha=1|\u) = 1$ 
if $\u$ satisfies $\alpha$ and $\pr(S_\alpha=1|\u) = 0$ otherwise.
By choosing \(p=0\), \(\e_1 = \{S_\alpha=1\}\) and \(\e_2 = \emptyset\), 
the D-Reverse-MAP query (is there $\u$ such that $\pr(S_\alpha=1|\u) > 0$) answers yes
iff there is an instantiation $\u$ that satisfies the formula 
$\alpha$ (SAT query). This concludes our proof.

\section{Review of Elimination Concepts}
\label{app:review}
We review here the standard notions of elimination process, clusters and moral graphs which we use in some of the upcoming proofs; see~\citep[Ch~9]{DarwicheBook09} for a detailed treatment. 

The \textit{moral graph} of an SCM $G$ is obtained from $G$ by adding an undirected edge between every pair of common parents and then undirecting all edges. 
\textit{Eliminating} a variable $X$ from a graph $G$ is done by connecting every pair of neighbors for $X$ in $G$, and then removing node $X$ from $G$.
Eliminating variables from an SCM \(G\) is done by
eliminating variables from its moral graph \(G'\).
Eliminating variables from a moral graph \(G'\)
using variable order \(\pi\) induces a \textit{graph sequence} \(G'=G_1, \ldots, G_n\) where graph \(G_{i+1}\) is obtained 
by eliminating variable \(\pi(i)\) from \(G_i\).
We use \(G_i(X)\) to denote \(X\) and its neighbors in graph \(G_i\).
We also use \(\C(X)\) to denote the \textit{cluster} of variable \(X\), which is $X$ and its neighbors just before eliminating \(X\). If $X = \pi(i)$, then $\C(X) = G_i(X)$.
We also use \(\C_i\) to denote the cluster for \(X\), $\C(X)$,
in this case.

\section{Proof of \lemmaref{lem:add-node}}
\label{app:lemma-add-node}
This proof uses the elimination concepts and notations reviewed in Appendix~\ref{app:review}.

Let $G'_m$ be the moral graph of $G'$ and $\C'_i$ be the cluster induces by eliminating variable $X_i$ from $G'$. Let $n$ be the number of variables in $G$. To prove 
Lemma~\ref{lem:add-node}, it suffices to prove the following statement:
$\C'_i \subseteq \C_i \cup \{H\}$ for $i = 1, \ldots, n,$
which we prove next by induction.

Let $\nei(X)$ denote the neighbors of $X$ in $G_m$ and let $\nei'(X)$ denote the neighbors of $X$ in $G'_m$. Let $\Z$ denote the children of $H$ in $G'$. For each $Z \in \Z$, let $\P_Z$ denote the parents of $Z$ in $G$. 
First, we show the statement holds when $i=1$. When creating $G'_m$ from $G'$, the introduction of node $H$ would cause two classes of edges that do not exist in $G_m$ to be added to $G'_m$: 1) $(Z, H)$ for $Z \in \Z$; 2) $(Y, H)$ if $Y$ is a parent of some node $Z \in \Z$, that is $Y$ and $H$ are common parents of some node $Z$. This means that before the elimination starts, for any node $X$, if $X \in \Z \cup_{Z \in \Z} P_Z$, we have $\nei'(X) = \nei(X) \cup \{H\}$; otherwise $\nei'(X) = \nei(X)$. 
Hence, $C'_1 \subseteq \{C_1\} \cup \{H\}$.
Consider now the elimination of $X_{i+1}$ assume that the 
statement holds for $1,2,\ldots,i$.
We observe that if $\C'_j \subseteq \C_j \cup \{H\}$, then the elimination of $X_j$ would cause only one type of additional edges be added to $G'_m$, that is $(Y, H)$ for $Y \in \nei(X_j)$. This is because the elimination of $X_j$ will form a clique among $\nei(X_j) \cup \{H\}$ in $G'_m$, but $\nei(X_j)$ already forms a clique after $X_j$ is eliminated from $G_m$. This implies that eliminating $X_j$ ($j \leq i$) will never cause any additional edge to be added among any two nodes that are both not $H$ (in other words, all additional edges added are incident on $H$). Thus, before we eliminate $X_{i+1}$, we have $\nei'(X_{i+1}) \subseteq \nei(X_{i+1}) \cup \{H\}$ and this implies $\C'_{i+1} \cap \C_{i+1} \cup \{H\}$ which concludes the proof.

\section{Proof of \theoremref{thm:objective-width}}
\label{app:objective width}
An \(N\)-world model is obtained by creating \(N\)
copies of a directed acyclic graph (DAG) while joining
them so a {\em subset} of their roots are shared~\citep{han2022on}.
Hence, an objective model as in
\definitionref{def:objective model} corresponds to an 
\(N\)-world model except for the addition of mixture node \(H\) 
and its outgoing edges. Before adding the mixture node \(H\),
an objective model with \(n\)
components corresponds to a \(3n\)-world model so
its treewidth is \(\leq 3n(w+1)-1\) by Theorem~\ref{thm:n-world}. 
After adding node \(H\), its treewidth is \(\leq 3n(w+1)\) by 
Lemma~\ref{lem:one node}.

\section{Proof of \theoremref{thm:constrained width bound}}
\label{app:constrained-tw}
This proof uses the elimination concepts and notations reviewed in Appendix~\ref{app:review}.

For a node \(X\) in an SCM $G$, we use $\dup{X}{k}$ to denote its
\(k^{th}\) duplicate in an $n$-world model of \(G\). If node \(X\)
is shared between all $n$ worlds, then $\dup{X}{k} = X$ for all $k$.
For a set of variables $\X$, we use $\dup{\X}{k}$ to denote 
$\{\dup{X}{k}: X \in \X\}$.

Let $G$ be an SCM, \(\U\) be a subset of its roots (unit variables) and let $G'$ be a corresponding objective model with $n$ components. Our proof is based on constructing an augmented objective model \(G''\) by adding edges to \(G'\) and then
showing that the bounds of Theorem~\ref{thm:constrained width bound} hold for \(G''\). Our proof is based on Lemmas~\ref{lem:constrained-width1} and~\ref{lem:constrained-width2} which we
formally state and prove later: 
\begin{itemize}
\item[--]Lemma~\ref{lem:constrained-width1} complements 
Theorem~\ref{thm:n-world} by showing that any \(\U\)-constrained elimination order
for an SCM can be converted into a \(\U\)-constrained elimination order for
a corresponding $n$-world model while preserving the width of the order.
\item[--]Lemma~\ref{lem:constrained-width2} concerns the
augmentation of an SCM by a root
node \(H\) and some edges that originate from \(H\). 
In particular, given a \(\U\)-constrained elimination order
of width \(w\) for the SCM, the lemma shows how to construct
a \(\U\)-constrained elimination order for its augmentation
with width \(\leq \max(w+1,|\U|)\).
\end{itemize}

We start by showing how to construct the augmented objective model
\(G''\) from \(G'\).
Let \(H\) be the mixture node of \(G'\) and
$\Z = \{\Y^i, \W^i\}_{i=1}^n$ be the set of all outcome variables in the objective function of Equation~\eqref{eqn:objective-function}. 
We obtain \(G''\) by adding to \(G'\)
an edge \(H \rightarrow Z\) for each \(Z \in \Z\) if such
an edge does not already exist in \(G'\).
The edges of $G''$ are a superset of the edges of $G'$ so it suffices to show that the bounds of Theorem~\ref{thm:constrained width bound} hold 
for \(G''\). 
We will next use \(G^t\) to denote a triplet ($3$-world) model
of \(G\). We will also use \(G^b\) to denote the augmentation 
of \(G^t\)
with mixture node \(H\) and edges \(H \rightarrow Z\) for
\(Z \in \Z\). Note that the augmented objective model \(G''\)
corresponds to \(n\) copies of \(G^b\) that share root 
nodes \(\U \cup \{H\}\). Hence, \(G''\) is an $n$-world model
of \(G^b\).
 
Let \(\pi\) be a \(\U\)-constrained elimination order for \(G\) with width \(w\).
Since $G^t$ is a triplet ($3$-world) model of \(G\), Theorem~\ref{thm:n-world}
tells us that there exists an elimination order \(\pi^t\) of \(G^t\) with width \(w^t\) such that $w^t \leq 3w +2$ (order \(\pi^t\) will also be \(\U\)-constrained). 
Recall that $G^b$ is obtained from \(G^t\)
by adding a root node $H$ and some edges that emanate from \(H\). 
By Lemma~\ref{lem:constrained-width2}, there exists an elimination
order \(\pi^b\) for \(G^b\) with width \(w^b\) such that $w^b \leq \max(w^t+1, |\U|) = \max(3w+3, |\U|)$. Moreover, if the objective function has a single outcome variable $Y$, then $H$ has a single child $Y$ in $G^b$ so, also by Lemma~\ref{lem:constrained-width2}, we have $w^b \leq 3w + 3$.
Since $G''$ is an $n$-world model of $G^b$ based on roots $\U\cup\{H\}$ of \(G^b\), we have $w'' = w^b$ by Lemma~\ref{lem:constrained-width1}. In summary, we have $w' \leq w'' \leq \max(3w + 3, |\U|)$. 
If the objective function has a single outcome variable, 
we have $w' \leq 3w + 3$. This concludes the
proof of Theorem~\ref{thm:constrained width bound}.

We will next formally state and prove Lemmas~\ref{lem:constrained-width1} and~\ref{lem:constrained-width2} which we used in the above proof. 
\begin{lemma}
\label{lem:constrained-width1}
Consider an SCM $G$, a subset \(\U\) of its roots, and a corresponding $n$-world model $G'$ for \(G\) that shares \(\U\).
If $w$ is the width of a $\U$-constrained elimination order \(\pi\) for $G$, and 
$w'$ is the width of the corresponding $\U$-constrained elimination order \(\pi'\) for $G'$, then \(w=w'\). 
\end{lemma}

Given an elimination order $\pi$ for an SCM $G$, we can convert it into a corresponding elimination order $\pi'$ for its $n$-world model $G'$ (referenced in the above lemma) by replacing each variable $X \notin \U$ in $\pi$ with its duplicates $\dup{X}{1}, \dup{X}{2}, \ldots, \dup{X}{n}$, as in Definition~2 in~\citep{han2022on}. If $\pi$ is $\U$-constrained, then $\pi'$ will also be $\U$-constrained. Moreover, we define a graph sequence for the $n$-world model $G'_1, G'_2, \ldots, G'_n$ where $G'_1$ is the moral graph of $G'$, and $G'_{i+1}$ is obtained by eliminating all duplicates of variable $\pi(i)$, i.e. $\dup{\pi(i)}{1}, \dup{\pi(i)}{2}, \ldots, \dup{\pi(i)}{n}$, from $G'_i$. \\

\begin{proof}
Suppose we eliminate variables from $G/G'$ using orders $\pi/\pi'$. 
We claim that at every elimination step $i$, the following properties hold:
\begin{enumerate}
\item[A)] For each node $X \notin \U$, $G'_i(\dup{X}{k}) = \dup{G_i(X)}{k}$
\item[B)] For each node $U \in \U$, $G'_i(U) = \union_{k=1}^{n} \dup{G_i(U)}{k}$
\end{enumerate}

We next show that properties A), B) imply $w' = w$ and then prove these properties. Let $Y = \pi(i)$.
If $Y \notin \U$, then when its duplicate $\dup{Y}{k}$ is eliminated from $G'$, we have $\C'(\dup{Y}{k}) = \dup{\C(Y)}{k}$. If $Y \in \U$, then when $Y$ is eliminated from $G'$, we have $\C'(Y) =  \union_{k=1}^{n} \dup{G_i(Y)}{k} = G_i(Y) = \C(Y)$ since all non-shared nodes have been eliminated before $Y$, i.e. $\dup{G_i(Y)}{k} = G_i(Y)$. This means that the cluster induced by eliminating a variable from $G'$ always has the same size as the cluster induced by eliminating the corresponding variable from $G$, which implies $w'=w$.

We next prove properties A), B) by induction.
By definition of an $n$-world model, these properties hold initially for $G'_1$. 
Suppose they hold for $G'_i$ and consider $G'_{i+1}$. 
Let $Y = \pi(i)$. Then $G'_{i+1}$ is the result of eliminating nodes $\dup{Y}{1},\ldots,\dup{Y}{n}$ from $G'_i$.  We consider two cases.

\noindent {\bf Case:} $Y \notin \U$. Consider each node $Z$ in $G'_{i}$. 
If $Z$ is not a neighbor of $\dup{Y}{1},\ldots,\dup{Y}{n}$ in $G'_i$, then $G'_{i+1}(Z)$ will not be affected by the elimination of $\dup{Y}{1},\ldots,\dup{Y}{n}$ and the properties hold by the induction hypothesis. Otherwise, node $Z$ falls into two cases: A) a duplicate $[X]^k$ of a node $X \notin \U$ B) a shared node $U \in \U$.
\begin{enumerate}
\item[A)] by the induction hypothesis, neighbors of $\dup{X}{k}$ in $G'_i$ must belong to the $k$-th world, so $\dup{X}{k}$ can only be a neighbor of the $k$-th duplicate $\dup{Y}{k}$. This means that $G'_{i+1}(\dup{X}{k})$ can only be affected by the elimination of $\dup{Y}{k}$. By definition of variable elimination, we have:
    \begin{align*}
    G'_{i+1}(\dup{X}{k}) &= G'_i(\dup{X}{k}) \cup G'_i(\dup{Y}{k}) \setminus \{\dup{Y}{k}\} \\
                &= \dup{G_i(X)}{k} \cup \dup{G_i(Y)}{k} \setminus \{\dup{Y}{k}\}   &\text{by the induction hypothesis}\\
                &= \dup{G_i(X) \cup G_i(Y) \setminus \{Y\}}{k} \\
                &= \dup{G_{i+1}(X)}{k} &\text{by definition of variable elimination}
    \end{align*}
This proves property A).

\item[B)] by the induction hypothesis, $U$ must be a neighbor of all duplicates $\dup{Y}{1},\ldots,\dup{Y}{n}$. We have:
    \begin{align*}
    G'_{i+1}(U) &= G'_{i}(U) \union_{k=1}^{n} G'_{i}(\dup{Y}{k}) \setminus \{ \dup{Y}{k}\}_{k=1}^{n} \\
                &= \big(\union_{k=1}^{n} \dup{G_i(U)}{k} \big) \big(\union_{k=1}^{n} \dup{G_{i}(Y)}{k} \big)\setminus \{ \dup{Y}{k} \}_{k=1}^{n} &\text{by the induction hypothesis} \\
                &= \union_{k=1}^{n}  \dup{G_i(U)}{k} \cup \dup{G_{i}(Y)}{k} \setminus \{\dup{Y}{k}\} \\
                &= \union_{k=1}^{n} \dup{G_i(U) \cup G_{i}(Y) \setminus \{Y\}}{k}  \\
                &= \union_{k=1}^{n} \dup{G_{i+1}(U)}{k} &\text{by definition of variable elimination}
    \end{align*}
This proves property B).
\end{enumerate}

\noindent {\bf Case:} $Y \in \U$. In this case, $G'_{i}$ only contains nodes in $\U$. Property A) holds trivially. And the relation in property B) reduces to $G'_i(U) = \union_{k=1}^{n} \dup{G_i(U)}{k} = G_{i}(U)$. By the induction hypothesis, we know $G'_i = G_{i}$ and thus $G'_{i+1} = G_{i+1}$. Property B) holds.
This concludes the proof.
\end{proof}

The proof of Lemma \ref{lem:constrained-width2} requires the following result on eliminating
variables from graphs.

\begin{lemma}
\label{lem:elimination-neighbor}
Consider a DAG $G$, a subset $\U$ of its nodes, and a node $H$ in \(G\) where $H \notin \U$. Let $G_1$ be the moral gragh of $G$, and $G_2$ be the result of eliminating all nodes other than $\{H\} \cup \U$ from $G_1$. For any node $X \in \U$, $X$ is adjacent to $H$ in $G_2$ if and only if there exists a path between $X$ and $H$ in $G_1$ that does not include a node in $\U \setminus \{X\}$.
\end{lemma}
\begin{proof}
We first prove the if direction. Suppose there exists such a path $(X,\ldots, Z_1, Y, Z_2, \ldots, H)$ in $G_1$. Eliminating node $Y$ from $G_1$ will lead to a path $(X,\ldots, Z_1,Z_2, \ldots, H)$. Since  nodes in $\U \setminus \{X\}$ cannot appear along this path, eliminating all nodes other than $\{H\} \cup \U$ will lead to the edge $(X, H)$ in $G_2$. We next prove the only-if direction by contraposition. Suppose there is no path between $X$ and $H$ in $G_1$ that does not include a node in $\U \setminus \{X\}$. There are two cases: 1) there is no path between $X$ and $H$; 2) every path between $X$ and $H$ includes at least one node $U \in U \setminus \{X\}$, which has the form $(X,\ldots, U, \ldots, H)$. In the first case, $X$ and $H$ will be disconnected in $G_2$. In the second case, eliminating all nodes other than $\{H\} \cup \U$ from such paths will lead to $X \xrightarrow{} U \xrightarrow{} H$, so $X$ cannot be directly adjacent to $H$ in $G_2$. This concludes the proof.
\end{proof}

\begin{lemma}
\label{lem:constrained-width2}
Consider an SCM $G$ and a subset $\U$ of its roots. Suppose SCM 
$G'$ is obtained from $G$ by adding a root node $H$ as a parent of some nodes $\Z$ in $G$ where $\Z\cap\U=\emptyset$. Let $\pi$ be a $\U$-constrained elimination order for $G$, and let $\pi'$ be a $\U$-constrained elimination order of $G'$ obtained from $\pi$ by placing $H$ just before variables $\U$. If $\pi$ has width $w$ and $\pi'$ has width $w'$, then $w' \leq \max(w+1, |\U|)$. Moreover, if $H$ has a single child in $G'$, then $w' = w+1$.
\end{lemma}
\begin{proof}
Let $\X$ denote variables other than $\U$ in $G$, and let $\U' = \U\cup\{H\}$.
Suppose we first eliminate variables $\X$, then $H$, and finally $\U$ from $G'$ using order \(\pi'\). This results in a graph sequence $G'_1, \ldots, G'_{j}, G'_{j+1/2}, G'_{j+1}, \ldots, G'_{j+k}$ where $j = |\X|$ and $k = |\U|$. Here, $G'_{j+1/2}$ is obtained by eliminating all variables $\X$ from $G'_1$, and $G'_{j+1}$ is obtained by eliminating $H$ from $G'_{j+1/2}$.
We claim:
\begin{itemize}
    \item if $i \leq j$, then for each node $X \neq H$ in $G'_i$, we have $G'_i(X) \subseteq G_i(X) \cup \{H\}$.
    \item if $i > j$, then for each node $X \neq H$ in $G'_i$, we have $G'_i(X) \subseteq \U'$. Moreover, if $H$ has a single child in $G'$, then $G'_i(X) = G_i(X)$.
\end{itemize}
We first show that the above claim implies the lemma, and then follow by proving the claim.
Suppose we are eliminating variable \(Y\) from \(G'\).
If $Y \notin \U'$ then \(i \leq j\) 
and the above claim implies $\C'(Y) \subseteq \C(Y)\cup\{H\}$.
If $Y \in \U'$ then \(i > j\) and the above claim implies $\C'(Y) \subseteq \U'$, and $\C'(Y) = \C(Y)$ when $H$ has a single child.
This guarantees the statement of the lemma:
$w' \leq \max(w+1, |\U|)$, and $w' = w+1$ if $H$ has a single child in $G'$.

We next prove our claim by induction. Let $\Z$ denote the children of $H$ in $G'$. When constructing the moral graph $G'_1$ from $G'$, the introduction of node $H$ causes two classes of edges that do not exist in $G_1$ to be added to $G'_1$: $(Z, H)$ for $Z \in \Z$, and $(Y, H)$ if $Y$ is a parent of some node $Z \in \Z$, that is $Y$ and $H$ are common parents of some node $Z$. All of these extra edges are incident on $H$, meaning that for any node $X$ in $G'_1$, $G'_1(X) \subseteq G_1(X) \cup \{H\}$. Thus, our claim holds for $G'_1$. Next, assume our claim holds for $G'_i$ (induction hypothesis) and consider $G'_{i+1}$.
We have two cases.

\noindent {\bf Case:} $i \leq j$. 
Let $Y = \pi(i)$. Consider each node $X$ in $G'_{i+1}$. If node $X$ is not a neighbor of $Y$ in $G_i/G'_i$, then $X$ is not affected by the elimination of $Y$, i.e., $G'_{i+1}(X) = G'_i(X)$ and $G_{i+1}(X) = G_i(X)$. So the claim holds by the induction hypothesis. Otherwise, we can bound $G'_{i+1}(X)$ as follows:
\begin{align*}
    G'_{i+1}(X) &= G'_i(X) \cup G'_i(Y) \setminus \{Y\} &\text{by the definition of  elimination} \\
                &\subseteq \big(G_i(X) \cup \{H\}\big) \cup  \big(G_i(Y) \cup \{H\}\big)  \setminus \{Y\} &\text{ by the induction hypothesis}  \\
                &\subseteq \big(G_i(X) \cup  G_i(Y) \setminus \{Y\}\big) \cup \{H\} \\
                &\subseteq G_{i+1}(X) \cup \{H\}          
\end{align*}

\noindent {\bf Case:}
$i > j$. For this case, $G'_{i}$ only contains nodes in $\U'$. It is trivial that $G'_i(X) \subseteq \U'$ for each node $X$ in $G'_{i}$. Recall that eliminating $H$ from $G'_{j+1/2}$ results in $G'_{j+1}$. By the induction hypothesis, all extra edges in $G'_{j+1/2}$ that do not exist in $G_{j+1}$ must be incident on $H$. Consider the special case where $H$ has a single child in $G'$. We claim that in this case, every two nodes in $G'_{j+1/2}(H)$ are adjacent in $G'_{j+1/2}$, meaning that the neighbors of $H$ already forms a clique in $G'_{j+1/2}$. Thus, eliminating $H$ from $G'_{j+1/2}$ will not add any fill-in edges in $G'_{j+1}$. This guarantees $G'_{j+1} = G_{j+1}$, i.e, $G'_i(X) = G_i(X)$ for all $i >= j+1$. We finally turn to proving this claim by contradiction. Suppose that node $U_1$ and $U_2$ are neighbors of $H$ in $G'_{j+1/2}$ but are not adjacent in $G'_{j+1/2}$. By Lemma~\ref{lem:elimination-neighbor}, in $G'_1$, there must be a path $P_1$ between $U_1$ and $H$ that does not include nodes in $\U \setminus \{U_1\}$, and a path $P_2$ between $U_2$ that does not include nodes in $\U \setminus \{U_2\}$. Since $H$ is a root and only has one child $Z$ in $G'$, $P_1$ must have the form $(U_1,\ldots, Z, H)$ in $G'_1$ and $P_2$ must have the form $(U_2,\ldots, Z, H)$ in $G'_1$. Thus, there must be a path $(U_1, \ldots, Z, \ldots, U_2)$ in $G'_1$ that does not contain nodes in $\U' \setminus \{U_1, U_2\}$. By Lemma~\ref{lem:elimination-neighbor}, after eliminating all nodes other than $U'$ from $G'_1$, $U_1$ and $U_2$ must be adjacent in $G'_{j+1}$. This leads to a contradiction.
\end{proof}

\section{Preliminary Experiment}
\label{app:experiment}
We provide next a preliminary experiment in which we compare the complexities of three algorithms:
(1)~MAP\_VE (Algorithm~\ref{alg:ve-map}) for computing MAP (operates on an SCM);
(2)~RMAP\_VE (Algorithm~\ref{alg:ve-rmap}) for solving unit selection (operates on an objective model);
and (3)~a baseline, bruteforce method for solving unit selection (operates on a twin-model).
We consider the complexities of these algorithms on
random SCMs generated using the method in~\citep{han2022on}. This method generates a random DAG and then ensures that each internal node in the DAG has at least one parent which is a root node by adding additional root nodes (to mimic the structure of SCMs).  
Such DAGs tend to have many root nodes and are particularly difficult for algorithms whose complexity is exponential in the constrained treewidth, like MAP\_VE and RMAP\_VE, as we show later.

Given a random SCM structure, we randomly select different percentages of roots
to be unit variables $\U$. We assume the objective function of~\citep{ijcai/LiP19} given in \equationref{eqn:benefit-ang}. This function has a single outcome variable which we choose randomly from the SCM leaves.
Moreover, as discussed earlier, this function requires only a twin model when constructing the objective model since it does not include evidence variables.
We do not prune the SCMs used by MAP\_VE, the objective models used by RMAP\_VE, or
the twin models used by the bruteforce method (see~\citep[Ch.~6]{DarwicheBook09}) so
the choice of interventional variables do not affect our complexity analysis (no evidence variables in the objective function of \equationref{eqn:benefit-ang}).
The time complexity of MAP\_VE is \(O(n\cdot \exp(w))\),
where \(n\) is the number of SCM nodes and \(w\) is the width of a \(\U\)-constrained elimination order for the SCM.
The time complexity of RMAP\_VE is \(O(n_1\cdot \exp(w_1))\), where $n_1$ is the number of nodes in the objective model and \(w_1\)
is the width of a $\U$-constrained elimination order for the objective model.
The bruteforce method enumerates every instantiation \(\u\)
and returns the one maximizing the objective \(L(\u)\). Its time complexity 
is \(O(n_2\cdot\exp(w_2))\), where $n_2$ is the number of nodes in the twin model used
to evaluate \(L(\u)\) and 
$w_2=|\U|+\!$
the width of an {\em unconstrained} elimination order for the twin model.
Hence, we compare the complexities of these three algorithms by reporting the number of nodes \(n\), \(n_1\), \(n_2\) and the corresponding widths \(w\), \(w_1\), \(w_2\). These are depicted in Table~\ref{tab:width} which also reports the number of SCM roots ($R$) and the percentage of roots used as unit variables ($ur$).

Before we highlight the outcomes of this experiment, we provide some insights into the class of used SCMs and their difficulty. We next characterize a class of problems for which the \(\U\)-constrained treewidth is no smaller than the number of unit variables, $|\U|$. The random SCMs we use in this experiment resemble this class of problems given how they are constructed.

\begin{definition}
\label{def:SCM-class}
    Consider a connected DAG $G$ and a subset $\U$ of its roots. We say that $\U$ are external to $G$
    if the DAG remains connected after removing nodes $\U$ and all their incident edges.
\end{definition}
Markovian SCMs (each root node has a single child) satisfy the above condition.

\begin{lemma}
    Consider a connected DAG $G$, a subset $\U$ of its roots, and its moral graph \(G'\). Let $\S$ be the subset of $\U$ such that for every two nodes $U_1$ and $U_2$ in $\S$, there exists a path between $U_1$ and $U_2$ in $G'$ that does not include any node in $\U\setminus\{U_1, U_2\}$. If $\pi$ is a $\U$-constrained elimination order of $G$ that has width $w$, then we have $w \geq |\S|$.
    \label{lem:constrain-width-roots}
\end{lemma}
\begin{proof}
    By Lemma~\ref{lem:elimination-neighbor}, every two nodes $U_1$ and $U_2$ in $\S$ will be adjacent after all nodes other than $\U$ are eliminated from $G'$. Thus, nodes in $\S$ will form a clique after all nodes other than $\U$ are eliminated. This leads to a cluster of size $|\S|$ during the elimination process, so $|\S|$ is a lower bound for the width of any \(\U\)-constrained elimination order. 
\end{proof}

Our main insight is stated in the following corollary which shows that MAP\_VE and RMAP\_VE
must be exponential in the number of unit variables for the class of SCMs (and unit variables) identified by Definition~\ref{def:SCM-class}. The baseline method can be significantly worse since it is exponential in the number of unit variables plus the unconstrained treewidth of the twin model.

\begin{corollary}
    Consider a connected SCM $G$, a subset $\U$ of its roots, and a $\U$-constrained elimination order $\pi$ with width $w$. If $\U$ are external to $G$, then $w \geq |\U|$.
    \label{cor:SCM-constrain-width-roots}
\end{corollary}
\begin{proof}
    Consider any two nodes $U_1$ and $U_2$ in $\U$. Suppose that $X_1$ is a child of $U_1$ and $X_2$ is a child of $U_2$. Since $\U$ are external to $G$, there exists a path $(X_1, \ldots, X_2)$ in the moral graph of $G$ that does not include any nodes in $\U$. Thus, there exists a path $(U_1, X_1, \ldots, X_2, U_2)$ that do not include any nodes in $\U \setminus \{U_1, U_2\}$. By Lemma ~\ref{lem:constrain-width-roots}, this implies $w \geq |\U|$ since $\S = \U$.
\end{proof}

We can now highlight the patterns in
Table~\ref{tab:width}. The complexities of MAP\_VE and RMAP\_VE are relatively close with the latter being more expensive than the former. 
Moreover, the gap between them narrows as the number of SCM variables ($n$) and the number of unit variables ($ur$) increase.
Note that according to Theorem~\ref{thm:objective-constrain-width}, $w_1/w \leq 2$ yet Table~\ref{tab:width} shows that this ratio can be significantly smaller than $2$.
Finally, the bruteforce method is significantly worse than RMAP\_VE and the gap between the two grows  as the number of SCM variables ($n$) and unit variables ($ur$) increase.

\begin{table}[tb]
\small
\centering
    \begin{tabular}{|c|c|c||c|c|c|c|c|c|c|c|c|c|c|c|}
        \hline
        \multicolumn{3}{|c||}{$ur$}
         & \multicolumn{4}{c|}{\text{20\%}} & \multicolumn{4}{c|}{\text{40\%}} & \multicolumn{4}{c|}{\text{60\%}}  \\
        \cline{1-15}
        $n$ &  $n_2$  & $R$  & $n_1$ & $w$ & $w_1$ & $w_2$  & $n_1$ & $w$ & $w_1$ & $w_2$  & $n_1$ & $w$ &$w_1$ & $w_2$  \\
        \hline
        10 & 14 & 6 & 
        52 & 5.5 & 7.2 & 7.3 
        & 49 & 5.5 & 7.4 & 8.3 
        & 46 & 5.5 & 7.5 & 9.3 \\
        \hline
        15 & 21 & 9 
        & 82 & 7.4 & 10.0 & 10.6 
        & 76 & 7.4 & 10.2 & 12.6 
        & 70 & 7.5 & 11.0 & 14.6 \\
        \hline
        20 & 30 & 12 
        & 116 & 10.1 & 14.0 & 16.0 
        & 110 & 10.1 & 14.5 & 18.0 
        & 101 & 10.1 & 15.4 & 21.0  \\
        \hline
        25 & 37 & 15 
        & 140 & 11.0 & 16.8 & 19.7 
        & 131 & 11.0 & 17.4 & 22.7 
        & 122 & 11.2 & 18.4 & 25.7  \\
        \hline
        30 & 43 & 17
        & 163 & 11.3 & 18.8 & 21.6
        & 154 & 11.4 & 19.2 & 24.6
        & 142 & 11.8 & 20.7 & 28.6 \\
        \hline
        35 & 50 & 19 & 
         190 & 12.4 & 21.4 & 24.2 &
         178 & 12.4 & 21.9 & 28.2 &
         166 & 12.8 & 23.5 & 32.2  \\
        \hline
        40 & 57 & 22 
        & 218 & 13.3 & 24.3 & 27.8
        & 206 & 13.6 & 24.9 & 31.8
        & 191 & 14.3 & 26.6 & 36.8 \\
        \hline
        45 & 64 & 24 
        & 246 & 14.3 & 26.6 & 29.8 
        & 231 & 14.4 & 27.2 & 34.8 
        & 216 & 15.6 & 28.9 & 39.8  \\
        \hline
\end{tabular}
\end{table}

\begin{table}
\small
\centering
\begin{tabular}{|c|c|c||c|c|c|c|c|c|c|c|}
    \hline 
    \multicolumn{3}{|c||}{$ur$} & \multicolumn{4}{c|}{\text{$80\%$}} & \multicolumn{4}{c|}{\text{$100\%$}}   \\
    \cline{1-11}
    $n$ & $n_2$ & $R$ & $n_1$ & $w$ & $w_1$ & $w_2$ & $n_1$ & $w$ & $w_1$ & $w_2$  \\
     \hline
        10 & 14 & 6 
        & 43 & 5.5 & 7.7 & 10.3 
        & 37 & 6.3 & 8.3 & 12.3 \\
    \hline
        15 & 21 & 9  
        & 64 & 8.2 & 11.8 & 16.6 
        & 58 & 9.6 & 12.6 & 18.6 \\
    \hline
        20 & 30 & 12
        & 95 & 10.5 & 15.6 & 23.0 
        & 86 & 12.8 & 16.4 & 26.0 \\
    \hline
        25 & 37 & 15
        & 113 & 12.8 & 18.8 & 28.7 
        & 104 & 15.9 & 19.9 & 31.7 \\
    \hline
        30 & 43 & 17 
        & 133 & 13.8 & 21.2 & 31.6
        & 121 & 18.0 & 21.6 & 35.6  \\
    \hline
        35 & 50 & 19
        & 154 & 15.6 & 24.0 & 36.2 
        & 142 & 19.6 & 23.6 & 40.2 \\
    \hline 
        40 & 57 & 22
        & 179 & 17.6 & 27.0 & 40.8
        & 164 & 23.0 & 26.6 & 45.8 \\
    \hline
        45 & 64 & 24
        & 201 & 20.1 & 30.6 & 44.8 
        & 186 & 25.6 & 29.2 & 49.8 \\
    \hline
\end{tabular}
\caption{Comparing 
the complexities of MAP\_VE for solving MAP ($n\exp(w)$),
RMAP\_VE for solving unit selection ($n_1\exp(w_1)$),
and the bruteforce method for solving unit selection ($n_2\exp(w_2)$). Each data point is an average over \(25\) runs. All elimination orders are computed using the minfill heuristic \citep{kjaerulff1990triangulation}.}
\label{tab:width}
\end{table}

\begin{figure}[tb]
    \centering
\includegraphics[width=90mm]{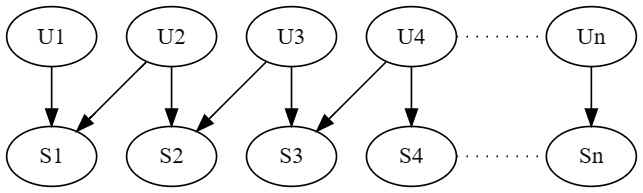}
    \caption{The unit variables are \(\U = U_1, \ldots, U_n\). The \(\U\)-constrained treewidth is $3$.}
\label{fig:bounded cw}
\end{figure}

We close this discussion by identifying a class of problems with an unbounded number of unit variables \(\U\) yet a bounded \(\U\)-constrained treewidth.
This class is depicted in Figure~\ref{fig:bounded cw}.
The $\U$-constrained treewidth is $3$, which can be shown using the $\U$-constrained elimination order
$S_1, \ldots,$
$ S_n, \ldots, U_1, \ldots, U_n.$
This is a class of problems for which unit selection using RMAP\_VE is tractable even when the number of unit variables is unbounded, assuming one uses a suitable objective function (e.g., the benefit function of~\citep{ijcai/LiP19} given in \equationref{eqn:benefit-ang}).

\end{document}